\documentclass{article}

\usepackage{microtype}
\usepackage{graphicx}
\usepackage{wrapfig}
\usepackage{booktabs} 
\usepackage{caption}
\usepackage{subcaption}
\usepackage{wrapfig}

\usepackage{amsthm}
\usepackage{amsmath}
\usepackage{amssymb}
\usepackage{dsfont}
\usepackage{url}  

\DeclareMathOperator*{\argmax}{arg\,max}

\newcommand{\norm}[1]{\left\lVert#1\right\rVert}

\usepackage{xspace}
\usepackage{xcolor}
\newcommand{\mr}[1]{}
\newcommand{\jp}[1]{}
\newcommand{\ml}[1]{}
\newcommand{\kt}[1]{}
\newcommand{\pfm}[1]{}
\newcommand{\mfg}[1]{}
\newcommand{\jc}[1]{}
\newcommand{\gp}[1]{}
\newcommand{\sarah}[1]{}

\usepackage{selectp}

\usepackage{hyperref}

\newtheorem{lemma}{Lemma}

\newtheorem{theorem}{Theorem}
\newtheorem{corollary}{Corollary}
\newtheorem{proposition}{Proposition}
\newtheorem{property}{Property}

\newcommand\Nash{NE\xspace}
\newcommand\MPMFG{MP-MFG\xspace}

\usepackage[accepted]{icml2021}

\icmltitlerunning{Online Mirror Descent for Mean Field Games}

\begin{document}

\twocolumn[
\icmltitle{Scaling up Mean Field Games with Online Mirror Descent}


\icmlsetsymbol{equal}{*}

\begin{icmlauthorlist}
\icmlauthor{Julien Perolat$^*$}{dm}
\icmlauthor{Sarah Perrin$^*$}{lille}
\icmlauthor{Romuald Elie$^*$}{dm}
\icmlauthor{Mathieu Lauri\`ere}{pr}
\icmlauthor{Georgios Piliouras}{sutd}
\icmlauthor{Matthieu Geist}{brain}
\icmlauthor{Karl Tuyls}{dm}
\icmlauthor{Olivier Pietquin}{brain}
\end{icmlauthorlist}

\icmlaffiliation{dm}{DeepMind Paris}
\icmlaffiliation{sutd}{Singapore University of Technology and Design}
\icmlaffiliation{pr}{Princeton University, ORFE}
\icmlaffiliation{brain}{Google Research, Brain Team}
\icmlaffiliation{lille}{Univ. Lille, CNRS, Inria, UMR 9189 CRIStAL}

\icmlcorrespondingauthor{Romuald Elie}{relie@google.com}

\icmlkeywords{Machine Learning, ICML}

\vskip 0.3in
]



\printAffiliationsAndNotice{\icmlEqualContribution} 

\begin{abstract}
We address scaling up equilibrium computation in Mean Field Games (MFGs) using Online Mirror Descent (OMD). We show that continuous-time OMD provably converges to a Nash equilibrium under a natural and well-motivated set of monotonicity assumptions. This theoretical result nicely extends to multi-population games and to settings involving common noise. A thorough experimental investigation on various single and multi-population MFGs shows that OMD outperforms traditional algorithms such as  Fictitious Play (FP). We empirically show that OMD scales up and converges significantly faster than FP by solving, for the first time to our knowledge, examples of MFGs with hundreds of billions states. This study establishes the state-of-the-art for learning in large-scale multi-agent and multi-population games.
\end{abstract}

\section{Introduction}

Solving decision making problems involving multiple agents has been the topic of intensive research in Artificial Intelligence for decades. It finds applications in a wide variety of domains such as economy~\cite{ConitzerS11,OthmanPRS13,achdou2014pde}, resource management~\cite{couillet2012electrical,FreedmanBSDC20}, crowd motion modeling~\cite{achdou2019mean} or even animal behaviour analysis~\cite{Phelps18,bardi2020convergence} among others. Despite the vast literature on Game Theory and numerous fundamental results, application to real-world problems remains a challenge. Recent successes of combining Game Theory and Machine Learning (especially Deep Learning~\cite{Goodfellow-et-al-2016} and Reinforcement Learning~\cite{Sutton2018}) led to solutions for large scale games such as chess~\cite{campbell2002deep}, Go~\cite{silver2016mastering,silver2017mastering,Silver18AlphaZero}, Poker~\cite{Brown17Libratus,Brown19Pluribus,moravvcik2017deepstack} and even complex video games like StarCraft II~\cite{vinyals2019grandmaster}. Although this allowed for tackling problems involving large states spaces, the number of agents remains still limited and scaling up to large populations of players remains intractable, which prevents a broader real-world impact.   

To address this challenge, the Mean Field Game (MFG) theory was introduced in~\cite{MR2295621,MR2346927-HuangCainesMalhame-2006-closedLoop} to study a category of games that involves an infinite population of agents. By considering the limit case of a continuous distribution of identical agents (\textit{i.e.}, anonymous  and with symmetric interests), the MFG framework allows the learning problem to be reduced to the characterization of the optimal behavior of a single representative agent in its interactions with the full population. Given this asymptotic formulation, traditional solutions to MFGs entail a coupled system of differential equations: one capturing the forward dynamics of the population and a second being the dynamic programming optimality equation of the representative player. Despite important progress in the area,
 such approaches are based on numerical approximation schemes for partial differential equations~\cite{MR2679575,achdou2012mean,MR3148086,MR3392626,MR3772008,BricenoAriasetalCEMRACS2017,achdoulauriere2020mfgnumerical} or for stochastic differential equations~\cite{chassagneux2019numerical,angiuli2019cemracs}, which are not easily scaled to large state spaces. Also, given the sensitivity to limit conditions, only simple configurations of the state space can be considered. So, until recently, we were left with solutions that either scale in terms of the state space dimension (deep RL) or in terms of large populations of agents (MFGs). Moreover, generalizations of the MFG framework to models with multiple populations have been introduced in~\cite{MR2346927} and have attracted a growing interest~\citep{MR3134900,MR3752669,MR3882529}. Applications include urban settlements~\cite{MR3597009} and crowd motion~\cite{LachapelleWolfram-2011-MFG-congestion-aversion,MR3763083}.
 
 By introducing solutions inspired by game theory (\textit{i.e.}, Fictitious Play~\cite{robinson1951iterative,shapiro1958}) into MFGs~\cite{cardaliaguet2017learning,elie2020convergence,perrin2020fictitious}, recent research leverages the generalization capacity of Machine Learning to compute a Nash equilibrium (\Nash) in large state spaces. Fictitious Play (FP) is a generic algorithm that alternates two steps starting from an arbitrary strategy for the representative player: i) computing the best response of this agent against the rest of the population, ii) compute the mixture of that best response with its previous strategy. \citet{perrin2020fictitious} propose to make use of most recent Reinforcement Learning (RL) methods to learn the best response and solve problems with millions of states with a non-trivial topology. Unfortunately, FP seems hard to scale further for several reasons. First, the computation of the best response remains a hard problem even if RL is promising. Second, its computational efficiency seems very low in practice. Finally, FP requires storing multiple quantities (\textit{e.g.} averaged policies and induced distributions, etc.) which contributes to cap scalability. 
 
 In this context, our first contribution is a new algorithm to compute a \Nash \textit{in lieu} of FP, namely Online Mirror Descent (OMD)~\cite{shalev2011online}. Inspired by convex optimization and the Mirror Descent algorithm~\cite{nemirovsky1979problem}, our method doesn't require the computation of a best response. It rather alternates a step of evaluation of the current strategy with a step of improvement of that strategy. The evaluation is done through the computation of the expected accumulated pay-offs of the strategy over time in the shape of a so-called $Q$-function. The improvement step reduces to computing the soft-max of the quantity obtained by integrating the $Q$-functions over iterations (like the MD algorithm suggests). Quantities that need to be stored by OMD (the strategy and the integrated $Q$-function) are thus limited compared to FP. As a second contribution, we provide a proof of convergence for continuous time OMD to a \Nash for MFGs under reasonable assumptions (common in the field). These theoretical results extend naturally to multi-population MFGs as well as settings where a noise is commonly shared by all agents. Our third contribution is an extensive empirical evaluation of OMD on different tasks involving single or multiple populations, in the presence of a common noise or not, with non trivial topologies. The scale of the considered problems reaches $10^{11}$ states and trillions of state-action pairs, surpassing by 4 or 5 orders of magnitudes existing results. These experiments demonstrate that OMD's computational efficiency is much stronger than FP which results in faster convergence. 
 
\section{Preliminaries on Mean Field Games}
In a Multi-Population Mean Field Game (\MPMFG), an infinite number of players from $N_p$ different populations interact with each other in a temporally and spatially extended game (the case $N_p=1$ corresponds to a standard MFG). Let $\mathcal{X}$ be the finite discrete state space and $\mathcal{A}$ be the finite discrete action space of the \MPMFG. We denote by  $\Delta \mathcal{X}$ and $\Delta \mathcal{A}$ respectively the spaces of probability distributions over states and actions. In this sequential decision problem, a representative player of population $i \in \{1, \dots, N_p\}$ starts at a state $x^i_0 \in \mathcal{X}$ according to a distribution $\mu^i_0 \in \Delta \mathcal{X}$. We consider a finite time horizon $N>0$. At each time step $n \in \{0, \dots, N\}$, the representative player of population $i$ is in state $x^i_n$ and takes an action according to $\pi^i_n(.|x^i_n)$, where $\pi^i_n \in  ({\Delta \mathcal{A}})^{\mathcal{X}}$ is a policy. Given this action $a^i_n$, the representative player moves to a next state $x^i_{n+1}$ with probability $p(.|x^i_n, a^i_n)$ and receives a reward $r^i(x^i_n, a^i_n, \mu^1_n, \dots, \mu^{N_p}_n)$, where $\mu^j_n$ is the distribution of the population $j$ at time $n$. Here $p \in ({\Delta \mathcal{X}})^{\mathcal{X} \times \mathcal{A}}$ and $r^i: \mathcal{X} \times \mathcal{A} \times (\Delta \mathcal{X})^{N_p} \to \mathbb{R}$. Observe that the transition kernel does not depend on the Multi-population distribution as in many classical MFG examples~\cite{MR2295621}.  For the reader's convenience, we denote $\pi^i = \{\pi^i_n\}_{n\in \{0, \dots, N\}}$, $\mu^i = \{\mu^i_n\}_{n\in \{0, \dots, N\}}$, $\pi = \{\pi^i\}_{i\in \{1, \dots, N_p\}}$, $\mu = \{\mu^i\}_{i\in \{1, \dots, N_p\}}$, $\pi_n = \{\pi^i_n\}_{i\in \{1, \dots, N_p\}}$ and $\mu_n = \{\mu^i_n\}_{i\in \{1, \dots, N_p\}}$.

During the game and for given a fixed multi-population distributions sequence $\mu$, a representative player of population $i$ accumulates the following sum of rewards:
\begin{align*}
    J^i(\pi^i, \mu) &= \mathbb{E}\Big[\sum \limits_{n=0}^N r^i(x^i_n, a^i_n, \mu_n) \;\Big|\; x^i_0 \sim \mu^i_0,\\
    &\qquad a^i_n \sim \pi^i_n(.|x^i_n), x^i_{n+1}\sim p(.|x^i_n, a^i_n)\Big].
\end{align*}
{\bf Backward Equation:} Given a population $i$, a time $n$, a state $x^i$, an action $a^i$, a policy $\pi^i$ and a multi-population distribution sequence $\mu$, we define the {\bf$Q$-function}:
\begin{align*}
    Q^{i,\pi^i, \mu}_n(x^i, a^i) &= \mathbb{E}\Big[\sum \limits_{k=n}^N r^i(x^i_k, a^i_k, \mu_k)|\; x^i_n = x^i, a^i_n = a^i,\\
    &\qquad a^i_k \sim \pi^i_k(.|x^i_k), x^i_{k+1}\sim p(.|x^i_k, a^i_k)\Big]
\end{align*}
and the {\bf value function}:
\begin{align*}
    V^{i,\pi^i, \mu}_n(x^i) &= \mathbb{E}\Big[\sum \limits_{k=n}^N r^i(x^i_k, a^i_k, \mu_k)|\; x^i_n = x^i,\\
    &\qquad a^i_k \sim \pi^i_k(.|x^i_k), x^i_{k+1}\sim p(.|x^i_k, a^i_k)\Big].
\end{align*}
These two quantities can be computed recursively with the following backward equations:
\begin{align*}
    &Q^{i,\pi^i, \mu}_N(x^i, a^i) = r^i(x^i, a^i, \mu_N)\\
    &Q^{i,\pi^i, \mu}_{n-1}(x^i, a^i) = r^i(x^i, a^i, \mu_{n-1})\\
    &\qquad\qquad+ \sum \limits_{{x'}^i\in \mathcal{X}} p({x'}^i|x^i, a^i) \mathbb{E}_{b^i\sim\pi^i_n(.|{x'}^i)}\Big[Q^{i,\pi^i, \mu}_{n}(x^i, b^i)\Big],
    \\
    &V^{i,\pi^i, \mu}_n(x^i) = \mathbb{E}_{a^i\sim\pi^i_n(.|{x'}^i)}\Big[Q^{i,\pi^i, \mu}_{n}(x^i, a^i)\Big].
\end{align*}
Finally, the sum of rewards is $J^i(\pi^i, \mu) = \mathbb{E}_{x^i \sim \mu^i_0}[V^{i,\pi^i, \mu}_n(x^i)]$.

{\bf Forward Equation:} If all the agents of a population $i$ follow the policy $\pi^i$, the distribution of the full population is defined recursively via the following forward equation: for all $x^i \in \mathcal{X}, \mu^{i,\pi^i}_0(x)=\mu^{i}_0(x)$
and for all ${x'}^i \in \mathcal{X},$ 
\begin{equation}
\label{eq:forwardeq}
\mu^{i,\pi^i}_{n+1}({x'}^i) = \sum \limits_{(x^i, a^i) \in \mathcal{X}\times \mathcal{A}} \pi^i_n(a^i|x^i)p({x'}^i|x^i,a^i)\mu^{i, \pi^i}_n(x^i)
\end{equation}
for $n \le N-1$. We denote $\mu^{\pi}=(\mu^{i,\pi^i})_{i \in \{1,\dots,N_p\}}$.

This leads to the following property for the cumulative sum of rewards $J^i(\pi^i, \mu) = \sum \limits_{n=0}^{N} \sum \limits_{(x^i, a^i) \in \mathcal{X}\times \mathcal{A}} \mu^{i, \pi^i}_n(x^i)\pi^i_n(a^i|x^i)r^i(x^i, a^i, \mu_{n})$.

{\bf Best Response and Exploitability:} A {\bf best response policy} $\pi^{i, br, \mu}$ to a multi-population distribution sequence $\mu$ verifies the following property $\max \limits_{\pi^i} J^i(\pi^i, \mu) = J^i(\pi^{i, br, \mu}, \mu)$. It can be computed recursively by finding the best responding $Q$-function $Q^{i,br, \mu}$:
\begin{align*}
    &Q^{i, br, \mu}_N(x^i, a^i) = r^i(x^i, a^i, \mu_N)\\
    &Q^{i, br, \mu}_{n-1}(x^i, a^i) = r^i(x^i, a^i, \mu_{n-1})\\
    &\qquad\qquad+ \sum \limits_{{x'}^i\in \mathcal{X}} p({x'}^i|x^i, a^i) \max_{b^i}\Big[Q^{i, br, \mu}_{n}(x^i, b^i)\Big].
\end{align*}
Finally $\pi_n^{i, br, \mu}(.|x^i) \in \argmax Q^{i, br, \mu}_{n}(x^i, .)$.

The {\bf exploitability} measures the distance to an equilibrum and is defined as $\phi(\pi) = \sum \limits_{i=1}^{N_p} \phi^i(\pi)$ where, for each $i$, $$\phi^i(\pi)=\max_{{\pi'}^i}J^i({\pi'}^i, \mu^\pi)-J^i(\pi^i, \mu^\pi).$$

{\bf Monotonicity:} A multi-population game is said to be {\bf weakly monotone} if for any $\rho^i_n, {\rho'}^i_n \in \Delta (\mathcal{X}\times \mathcal{A})$ and $\mu^i_n, {\mu'}^i_n \in \Delta \mathcal{X}$ such that for all $i, n, x^i \;  \mu^i_n(x^i) = \sum \limits_{a^i \in \mathcal{A}} \rho^i_n(x^i, a^i) \textrm{ and } {\mu'}^i_n(x^i) = \sum \limits_{a^i \in \mathcal{A}} {\rho'}^i_n(x^i, a^i)$, we have:
$\sum \limits_i \sum \limits_{(x^i, a^i) \in \mathcal{X}\times \mathcal{A}} (\rho^i_n(x^i, a^i)-{\rho'}^i_n(x^i, a^i))(r^i(x^i, a^i, \mu_n)-r^i(x^i, a^i, {\mu'}_n))\leq 0$.
It is {\bf strictly weakly monotone} if the inequality is strict whenever $\rho_n \neq  \rho'_n$. This condition means that the players are discouraged from taking similar state-action pairs as the rest of the population. Intuitively,  it can be interpreted as an aversion to crowded areas.

We have the following consequence, which is enough to derive many properties.
\begin{lemma}
\label{lem:wmon-tildeM}
The weak monotonicity property implies that for any  $\pi, {\pi'}$ with $\pi \neq \pi'$,
\begin{align}
\tilde{\mathcal{M}}(\pi,\pi') &:= \sum \limits_{i=1}^{N_p}\big[J^i(\pi^i, \mu^{\pi}) +  J^i({\pi'}^i, \mu^{\pi'}) \notag\\
&\qquad\quad - J^i(\pi^i, \mu^{\pi'}) - J^i({\pi'}^i, \mu^{\pi}) \big]\leq  0. \label{Weak_Monotonicity}
\end{align}
Strictly weak monotonicity implies a strict inequality above. 
\end{lemma}
This result is proved in Appendix~\ref{app:wmon-tildeM}.

Moreover, the weak monotonicity condition is met for example in the following classical framework.
\begin{lemma}\label{lem:sep-mon-wmon}
Assume the reward is {\bf separable}, i.e.  $r^i(x^i, a^i, \mu) = \bar r^i(x^i, a^i) + \tilde r^i(x^i, \mu)$ and the following {\bf monotonicity condition} holds: for all $\mu \neq \mu',$ $\sum \limits_i \sum \limits_{x \in \mathcal{X}} (\mu^i(x^i)-{\mu'}^i(x^i))(\tilde r^i(x^i, \mu)-\tilde r^i(x^i, \mu'))\leq 0$ (resp. $<0$). Then the game is weakly monotone (resp. strictly weakly monotone).
\end{lemma}
This result is proven in Appendix~\ref{Separable_Monotonicity_imply_WMonotonicity}.

An example of such a separable and monotone reward can be found in multi-population predator prey models where the reward can be expressed as a network zero-sum game:
\begin{align}
    &r^i(x^i, a^i, \mu) 
    \notag
    \\
    &= \bar r^i(x^i, a^i) + \underbrace{\hat r^i(x^i, \mu^i) + \sum_{j\neq i}\mu^j(x^i)\hat r^{i,j}(x^i)}_{=\tilde r^i(x^i,\mu)}
\label{eq:reward_multi}
\end{align}
if $\forall x\in\mathcal{X}, \hat{r}^{i,j}(x) = - \hat{r}^{j, i}(x)$ and if $\forall \mu \neq \mu', \forall i, \sum \limits_{x \in \mathcal{X}} \Big(\mu^i(x^i)-{\mu'}^i(x^i)\Big)\Big(\hat r^i(x^i, \mu^i)-\hat r^i(x^i, {\mu'}^i)\Big)\leq 0$ (or with a strict inequality).

{\bf Nash Equilibrium (\Nash):} A \Nash is a vector of policies for all populations that has a $0$ exploitability. The existence of a \Nash in MFGs has been studied in many settings~ \cite{cardaliaguet2010notes,MR3134900,carmona2018probabilisticI-II}. In our framework, it is a consequence of the convergence of the Fictitious play dynamics in monotone games which will be introduced later and proved in Appendix~\ref{fp_proof}.

\begin{proposition}[Existence and uniqueness of Nash]\label{lem:swm-uniqness}
    Any weakly monotone \MPMFG admits a \Nash. Besides, if the weak monotonicity is strict, the \Nash is unique.
\end{proposition}
\begin{proof} The existence result follows from Theorem~\eqref{thm:fp_FH}, while the uniqueness property is proven in  Appendix~\ref{app:swm-uniqueness}.\end{proof}

\section{Background on Fictitious Play}
One can extend the Fictitious Play work of~\citeauthor{perrin2020fictitious} to a multi-population setting. In the Multi-Population case, the Fictitious play process is defined as follows. Let first picking $1$ as an arbitrary but classical reference time. For $t<1$, we consider a fixed uniform policy for all representative player $i$ at all time-step $n$ denoted $\pi^{i,br}_{n,t<1}$ and inducing a distribution $\mu^{i,br}_{n,t}$. We define  $\forall t\geq 1$ the distribution $\mu^{i}_{n,t}$ as:
\begin{align*}
    \forall i, n,\; \mu^{i}_{n,t}(x^i) = \frac{1}{t} \int \limits_{s=0}^t \mu^{i,br}_{n,s} (x^i)ds\,,
\end{align*}
where, for all $t\geq 1$, $\mu^{i,br}_{n,t}$ is the distribution of a best response policy $\pi^{i,br}_{n,t}$ to $\mu^{i}_{n,t}(x^i)$.
The policy $\pi^i_{n,t}$ of the distribution $\mu^{i}_{n,t}$ verifies the following equation (see~\citeauthor{perrin2020fictitious}): for all $i, n, x^i, a^i$,
\begin{align*}
    &\pi^i_{n,t}(a^i|x^i)\int \limits_{s=0}^t \mu^{i,br}_{n,s} (x^i) ds = \int \limits_{s=0}^t \pi^{i,br}_{n,s}(a^i|x^i) \mu^{i,br}_{n,s} (x^i) ds
\end{align*}

\begin{theorem}
\label{thm:fp_FH}
If a \MPMFG satisfies the weak monotony assumption, the exploitability is a strong Lyapunov function of the Fictitious Play dynamical system, $\forall t \geq 1$:
$\frac{d}{dt} \phi(\pi^t) \leq - \frac{1}{t} \phi(\pi^t).$
Hence $\phi(\pi^t) = O(\frac{1}{t})$.
\end{theorem}
\begin{proof} This is an extension to multi-population of Theorem~1 of \cite{perrin2020fictitious}.
The full proof is left in Appendix~\ref{fp_proof}.
\end{proof}

\section{Online Mirror Descent: algorithm and convergence result}
We now turn to the Online Mirror Descent Algorithm and introduce a regularizer $h: \Delta \mathcal{A} \rightarrow \mathbb{R}$, that is assumed to be $\rho$-strongly convex for some constant $\rho>0$.
Furthermore, we will assume from this point forward that the regularizer 
 $h$ is \emph{steep}, i.e., $\norm{\nabla h(\pi)} \to \infty$ whenever $\pi$ approaches the   of $\Delta \mathcal{A}$; The classic negentropy regularizer, which results to replicator dynamics is the prototypical example of this class.
Denote by $h^*: \mathbb{R}^{|\mathcal{A}|} \rightarrow \mathbb{R}$ its convex conjugate defined by $h^*(y) = \max \limits_{p \in \Delta\mathcal{A}} [\langle y,p \rangle - h(\pi)]$. Since $h$ is differentiable almost everywhere, we have, for almost every $y$,   
\begin{equation}
\label{eq:Dhstar-Gamma}
    \Gamma(y) 
    := \nabla h^*(y) 
    = \argmax_{p\in\Delta\mathcal{A}} [\langle y,p \rangle - h(\pi)].
\end{equation}
{\bf Discrete Time Online Mirror Descent: } The OMD algorithm is implemented as described in Algorithm~\ref{Algo_OMD}. At each iteration, the first step consists in computing, for each population, the evolution of the population's distribution by using the current policy, see~\eqref{eq:forwardeq}. In the second step, each population's policy is updated. This update is done by first updating the corresponding $y$ variable and then obtaining the policy thanks to the function $\Gamma$. We have
for all $t > 0, i\in \{1,\dots,N_p\}, n\in \{0,\dots,N\}$,
\begin{align*}
    y^i_{n,t+1}(x^i,a^i) 
    &= \sum \limits_{s=0}^{t} \alpha Q^{i,\pi^i_s, \mu^{\pi_s}}_n(x^i,a^i),
    \\
    \pi^i_{n,t+1}(.|x^i) 
    &= \Gamma(y^i_{n,t+1}(x^i,.)).
\end{align*}
\begin{algorithm}[tb]
   \caption{Online Mirror Descent (OMD)}
   \label{Algo_OMD}
\begin{algorithmic}
   \STATE {\bfseries Input:} $\alpha, y^{i}_{n,0}=0$ for all $i,n$; $t_{max}$.
   \REPEAT
   \STATE Forward Update: Compute for all $i$, $\mu^{i, \pi^{(t)}}$
   \STATE Backward Update: Compute for all $i$, $Q^{i, \pi^{i}_t}, \mu^{\pi_t}$
      \STATE {Update for all $i,n, x, a$, 
   \begin{align*}
   y^{i}_{n,t+1}(x,a) 
   &= y^{i}_{n,t}(x,a) + \alpha Q^{i, \pi^{i}_t, \mu^{\pi_t}}_n(x, a)
   \\
   \pi^{i}_{n,t+1}(. |x) 
   &= \Gamma(y^{i}_{n,t+1}(x,.))
   \end{align*}
   }
   \UNTIL{$t=t_{max}$}
\end{algorithmic}
\end{algorithm}
{\bf Continuous Time Online Mirror Descent: } We study the theoretical convergence of the continuous time version of Algorithm~\ref{Algo_OMD}. Namely, the Continuous Time Online Mirror Descent (CTOMD) algorithm \cite{mertikopoulos2017cycles} is defined as: for all $i\in \{1,\dots,N_p\}, n\in \{0,\dots,N\}$, $y^i_{n,0} = 0$, and for all $t \in \mathbb{R}_+$,
\begin{align}
    y^i_{n,t}(x^i,a^i) &= \int \limits_{0}^{t} Q^{i,\pi^i_s, \mu^{\pi_s}}_n(x^i,a^i) ds,
    \label{eq:OMD-y}
    \\
    \pi^i_{n,t}(.|x^i) &= \Gamma(y^i_{n,t}(x^i,.)).
    \label{eq:OMD-pi}
\end{align}

From here on, unless otherwise specified, we assume that the weak monotonicity condition holds and denote by $\pi^*$ a \Nash, whose existence follows from Proposition \ref{lem:swm-uniqness}. We let $y^{i,*}: (x^i,a^i) \mapsto y^{i,*}(x^i,a^i)$ be the corresponding dual variable such that $\pi^{i,*}(.|x^i) = \Gamma(y^{i,*}(x^i,.))$ for every $i$.

{\bf Measure of similarity with the \Nash $\pi^*$:} 
 Based on the regularizer $h$, we define in the dual space the following measure of similarity $H: \mathbb{R}^{|\mathcal{A}|} \to \mathbb{R}$ with the \Nash $\pi^*$:
\begin{align*}
    H(y) 
    &:= \sum_{i=1}^{N_p} \sum \limits_{n=0}^N \sum \limits_{x^i \in \mathcal{X}} \mu^{i,\pi^*}_n(x^i)\Big[h^*(y^i_{n,t}(x^i,.))
    \\
    & - h^*(y^{i,*}(x^i,.)) - \langle \pi^{i,*}_{n,t},y^i_{n,t}(x^i,.)-y^{i,*}_{n,t}(x^i,.) \rangle \Big].
\end{align*}

As detailed below, this quantity will be decreasing through the iterations of CTOMD. Observe that since the regularizer is steep and thus always maps in the interior of the simplex, it can also be expressed in terms of Bregman divergence as:

\begin{align*}
    H(y) 
    &=\sum_{i=1}^{N_p} \sum \limits_{n=0}^N \sum \limits_{x^i \in \mathcal{X}} \mu^{i,\pi^*}_n(x^i)[D_{h}(\pi^{i,*}_n(x^i,\cdot), \pi^i_n(x^i,\cdot))].
\end{align*}

which is always non-negative. 
Here $D_{F}$ denotes the Bregman divergence associated with a map $F$ and defined as    $D_{F}(p,q) := F(p) - F(q) - \langle \nabla F(q), p-q \rangle.$ In this derivation we have used
  known relations between Fenchel couplings and Bregman divergences (e.g.,~\citet{mertikopoulos2016learning})
  and denoted $\pi^i_n:=\Gamma(y^i_n)$.  Thus, the similarity measure $H$ can also be expressed in terms of proximity between policies.  

We are now in position to characterize the dynamics of the similarity to the Nash mapping via the following lemma, whose proof is provided in Appendix \ref{sec:app-OMD}.

\begin{lemma}[Similarity to Nash dynamics]\label{Lemma_Similarity}
In CTOMD, the measure of similarity $H$ to the Nash $\pi^*$ satisfies
\begin{align}
    \frac{d}{dt} H(y_t) 
    &= \Delta J(\pi_t, \pi^*)  + \tilde {\mathcal{M}}(\pi_t, \pi^*) 
\end{align}
where $\Delta J(\pi_t, \pi^*)  := \sum_{i=1}^{N_p}J^i(\pi^i_t,\mu^{\pi^*}) -  J^i(\pi^{i,*}, \mu^{\pi^*})$ is always non-positive, and where the weak monotonicity metric $\tilde{\mathcal{M}}$ is defined in \eqref{Weak_Monotonicity}.
\end{lemma} 

{\bf Convergence to the Nash for \MPMFG{}s:} We now turn to the main theoretical contribution of the paper, by deriving the convergence of CTOMD to the set of \Nash for \MPMFG{}s.

\begin{theorem}[Convergence of CTOMD]
\label{thm_convergence}
If a \MPMFG satisfies $\tilde {\mathcal{M}}(\pi, \pi')<0$ if $\mu^{\pi}\neq \mu^{\pi'}$ and 0 otherwise, then $(\pi_t)_{t \ge 0}$ generated by CTOMD given in~\eqref{eq:OMD-pi} converges to the set of Nash equilibria of the game as $t \to +\infty$. 
\end{theorem}

\begin{proof} 
The proof is left in appendix~\ref{proof_lyapunov}.
\end{proof}
Thanks to Lemma~\ref{lem:wmon-tildeM} together with Proposition \ref{lem:swm-uniqness}, we easily deduce the convergence to the unique \Nash in some more stringent classes of \MPMFG{}s. 
\begin{corollary}[Convergence of CTOMD for weakly monotone MFG]
For any strictly weakly monotone \MPMFG, $(\pi_t)_{t \ge 0}$ generated by CTOMD given in~\eqref{eq:OMD-pi} converges to the unique \Nash of the game, as $t \to +\infty$. 
\end{corollary}

\begin{corollary}[Convergence of CTOMD for multi-population network zero sum MFG]\label{Cor_Multi_Pop}
For any strictly monotone and essentially zero-sum \MPMFG, $(\pi_t)_{t \ge 0}$ generated by CTOMD given in~\eqref{eq:OMD-pi} converges to the unique \Nash of the game, as $t \to +\infty$.
\end{corollary} 

It is worth noticing that the argumentation followed in our proof differs from the usual approaches on regret minimization arguments as e.g. in \cite{zinkevich2008regret}. 

{\bf Restriction to a single population MFG:}
Finally, considering the number of populations $N_p$ equals $1$, we deduce a convergence result of CTOMD to the \Nash of single population for strictly weakly monotone MFG.  

\begin{corollary}[Convergence of CTOMD for Single Population MFG]\label{Theorem_Single_Pop}
For any single population MFG satisfying the strictly weak monotonicity assumption, $(\pi^{(t)})_{t \ge 0}$ generated by CTOMD given in~\eqref{eq:OMD-pi} converges to the unique \Nash of the game, as $t \to +\infty$. 
\end{corollary}

\begin{table*}[t]
\begin{center}
 \begin{tabular}{|c | c c c c|} 
 \hline
  Environment & $|\mathcal{X}|$ & $|\mathcal{X}|\times|\mathcal{A}|$ & OMD & FP  \\ 
 \hline\hline
 Garnet & $2\times10^{3}$ -- $2\times10^{4}$ & $2\times10^{4}$ -- $4\times10^{5}$ & $84$Ko -- $229$Ko & $168$Ko -- $458$Ko \\ 
 \hline
 Building & $8\times 10^{9}$ & $5.6\times 10^{10}$ & $0.21$To & $0.42$To \\ 
 \hline
 Common noise & $2.73\times10^{11}$ & $1.092\times10^{12}$ & $5.0$To & $10$To \\
 \hline
 Multi-Population medium & $5\times10^{7}$ & $2\times10^{8}$ & $0.93$Go & $1.9$Go \\
 \hline
 Multi-Population large & $8\times10^{8}$ & $3.2\times10^{9}$ & $73$Go & $146$Go \\
 \hline
 \end{tabular}

\caption{Number of states, action-states pairs \& RAM memory required for the experiments. $|\mathcal{X}| = \text{positions} \times \text{timesteps} \times \text{common noise} \times \text{number of populations})$.\label{table:nb_states}}
\end{center}
\end{table*}

\section{Numerical experiments}

We illustrate the theoretical convergence of CTOMD with an extensive empirical evaluation of OMD described in Algorithm \ref{Algo_OMD} within various settings involving single or multiple populations as well as non trivial topologies (videos available \href{https://www.youtube.com/channel/UCsyraNAh_zmwvChXLKZo_Eg}{here}). These settings are typically hardly tractable using classical numerical approximation schemes for partial differential equations. Besides, the scale of the numerical experiments grows up to $10^{12}$ states, establishing a new scalability benchmark in the MFG literature. We emphasize the diversity of tractable environments by considering (randomized MDP) Garnet settings, a twenty-storey high building evacuation, a crowd movement example in the presence of common noise and finally an essentially zero sum multi-population chasing game.

{\bf Experimental setup:} We compare OMD and FP with different learning rates $\alpha$. In discrete-time OMD, $\alpha$ appears in the backward update of $y$: $y^{i}_{n,t+1}(x,a) = y^{i}_{n,t}(x,a) + \alpha Q^{i, \pi^{i}_t, \mu^{\pi_t}}_n(x, a)$, whereas in discrete-time FP, it corresponds to how much we update the average policy with the new best response $\pi^{i}_{n,t+1}(x^i,a^i)$ given by $$\frac{ (1-\alpha_t) \mu^{i, \pi_t}_{n}(x^i) \pi^{i}_{n,t+1}(x^i,a^i) + \alpha_t \mu^{i, br}_{n,t}(x^i)\pi^{i,br}_{n,t}(x^i,a^i)}{(1-\alpha_t) \mu^{i, \pi_t}_{n}(x^i) + \alpha_t \mu^{i, br}_{n,t}(x^i)}.$$ FP is experimented with decreasing  $\alpha_t=\alpha/(2+t)$ or constant $\alpha_t=\alpha$ learning rate. This latter is referred to hereafter as \textit{FP damped}, while $\alpha=1$ corresponds to the fixed point iteration algorithm, \textit{i.e.} the population applies the last best response policy. The theoretical proof of convergence relies on restrictive conditions which only hold for a small class of games. We provide a thorough evaluation in Table~\ref{table:nb_states} of the complexity of the environments along with the memory required to compute our results. For OMD, we only need to store $y$ of size $|\mathcal{X}| \times |\mathcal{A}|$ and the distributions, of size $|\mathcal{X}|$. For FP, we need to store the last best response, the average policy, the last distribution and the average distribution, requiring a total of $2\times (|\mathcal{X}| \times |\mathcal{A}|) + 2\times |\mathcal{X}|$. In all the experiments, $h$ is the entropy: $h = -\sum_{a \in \mathcal{A}} \pi(a)\log(\pi(a))$. This implies that $h^*(y) = \log(\sum_a \exp(y(a)))$, and we find that $\Gamma$ is a softmax if we take the gradient of $h^*$.

\begin{figure*}[htbp]
    \centering
    \includegraphics[width=1\linewidth]{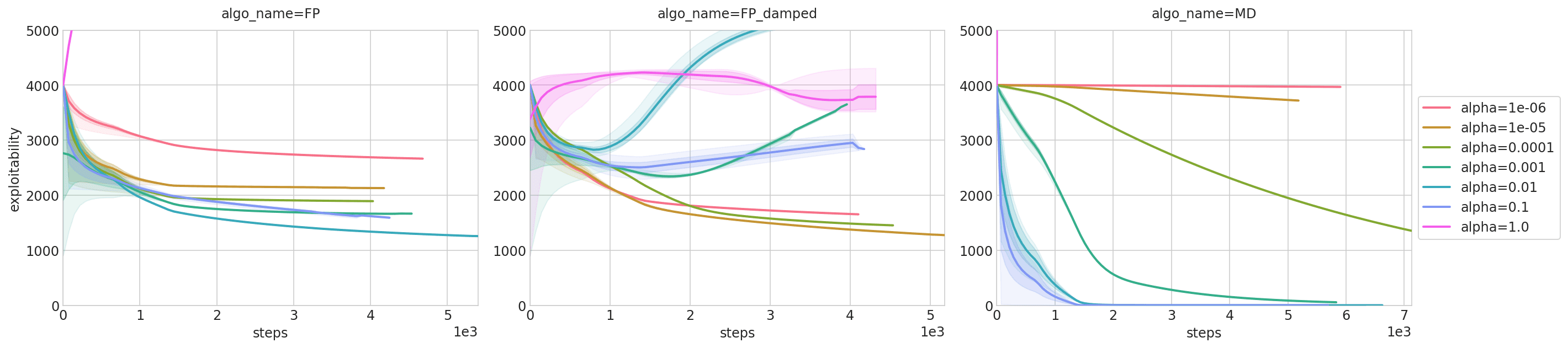}
    \vspace{-2em}
    \caption{5 Garnet sampled with param $n_x=20000$, $n_a=10$, $t=2000$, $s_f=10$} 
  \label{fig:main_garnet_plot} 
\end{figure*}
    
\subsection{Garnet}

We first evaluate Alg.~\ref{Algo_OMD} on a set of randomly generated problems (repeatability of our results for varying sizes).

{\bf Environment: } A garnet is an abstract and randomly generated MDP~\citep{archibald1995generation}. We adapt this concept to single-population MFGs by modifying the reward. In our case, a Garnet is built from the set of parameters $(n_x, n_a, n_b, s_f, \eta)$, with $n_x$ and $n_a$ respectively the numbers of states and actions. The term $n_b$ is a branching factor, and the transition kernel (independent of $\mu$) is built as follows: $n_b$ transiting states are drawn randomly without replacement, and the associated transition probabilities are obtained by partitioning the unit interval with $n_x-1$ uniformly sampled random points. The reward term $\tilde{r}(x,u)$ is set to 0 for $s_f$ states sampled randomly without replacement, for each of the remaining states it is set for all actions to a random value sampled uniformly in the unit interval.
We set $\bar{r}(s,\mu) = - \eta \log(\mu(x))$. This reward encourages the agents to spread out accross the MDP states and can model social distancing. This process generates a monotone MFG.

{\bf Numerical results: } 
Fig.~\ref{fig:main_garnet_plot} (main text) and~\ref{garnet_plot} (Appx.~\ref{subappx:garnet}) shows various Garnet experiments. We fix $s_f=10$, $t=2000$, $\eta=1$ and $n_b=1$ (deterministic dynamics) and vary $n_x\in\{2.10^3;2.10^4\}$ and $n_a\in\{10,20\}$. In each case, results are averaged over 5 randomly generated Garnets. We compare OMD to FP, damped or not. We observe that OMD consistently converges faster for the right choice of $\alpha$.  $\alpha=1$ might lead to unstable results while $\alpha=0.1$ consistently provides fast convergence to the Nash. In all cases, the number of states influences the convergence rate, but much less for OMD.

\begin{figure*}[htbp]
    \centering
    \begin{subfigure}{0.23\textwidth}
        \centering
        \includegraphics[width=1.0\linewidth]{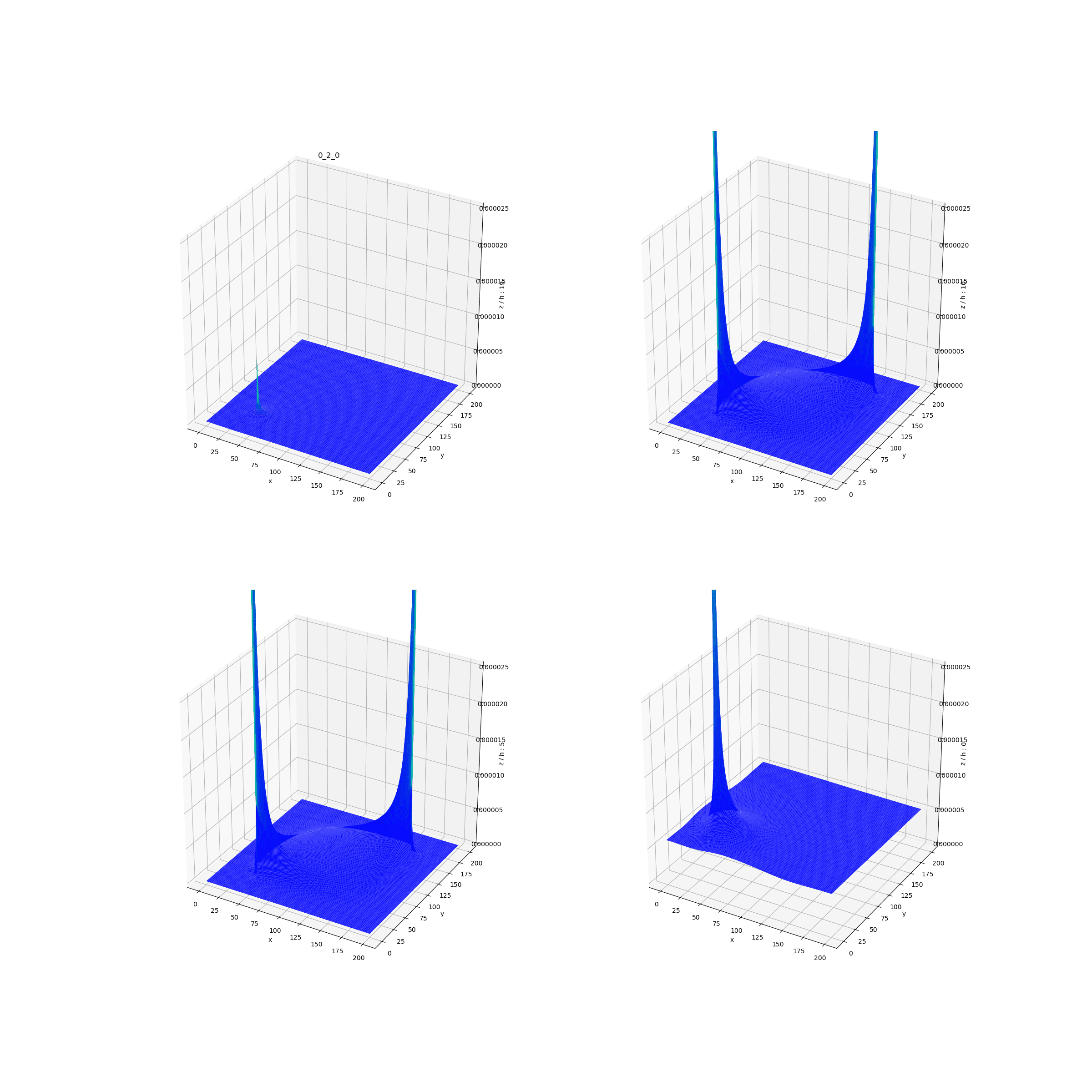}
        \caption{After $1000$ timesteps.}
    \end{subfigure}
    \begin{subfigure}{0.23\textwidth}
        \centering
        \includegraphics[width=1.0\linewidth]{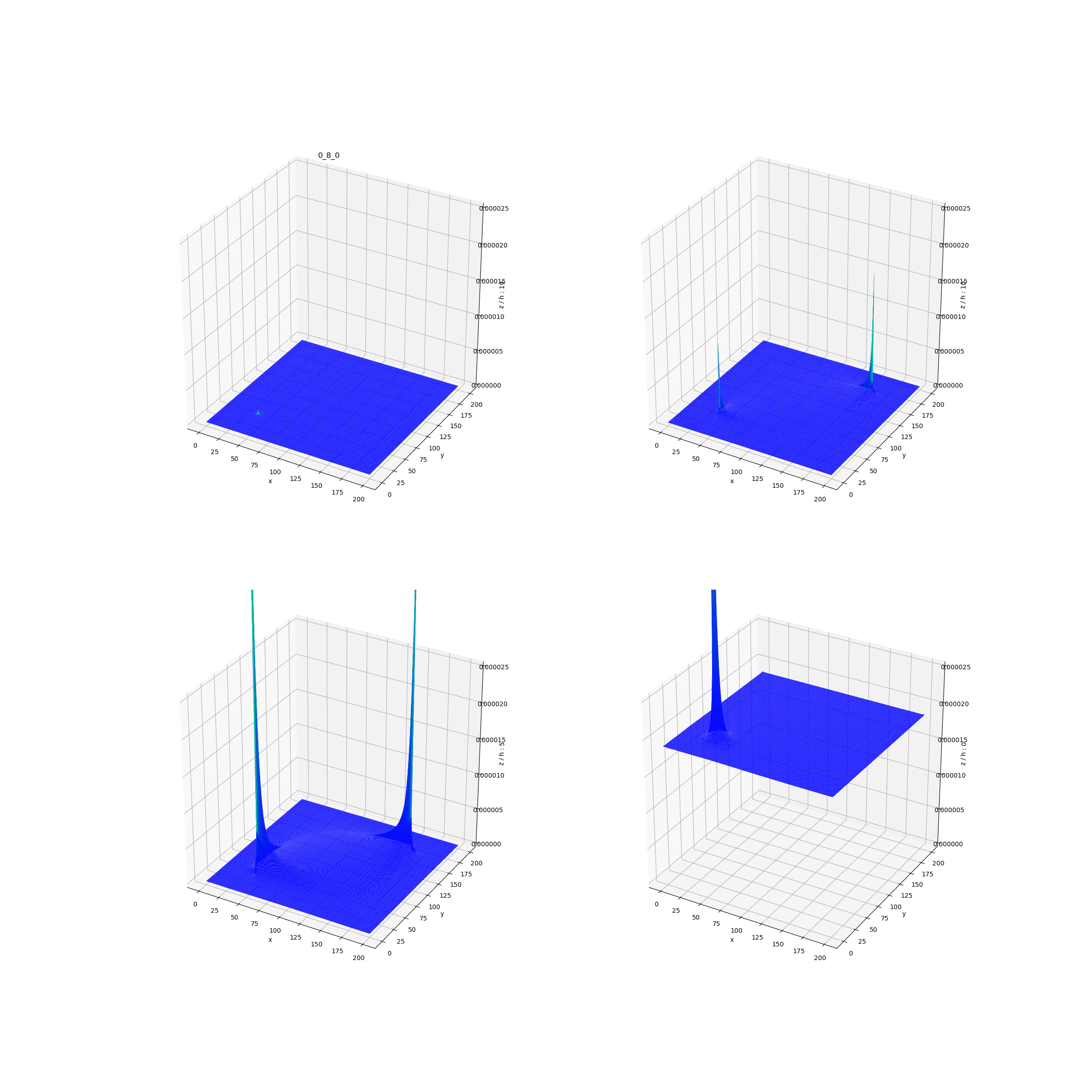}
        \caption{After $4000$ timesteps.}
    \end{subfigure}
    \begin{subfigure}{0.23\textwidth}
        \centering
        \includegraphics[width=1.0\linewidth]{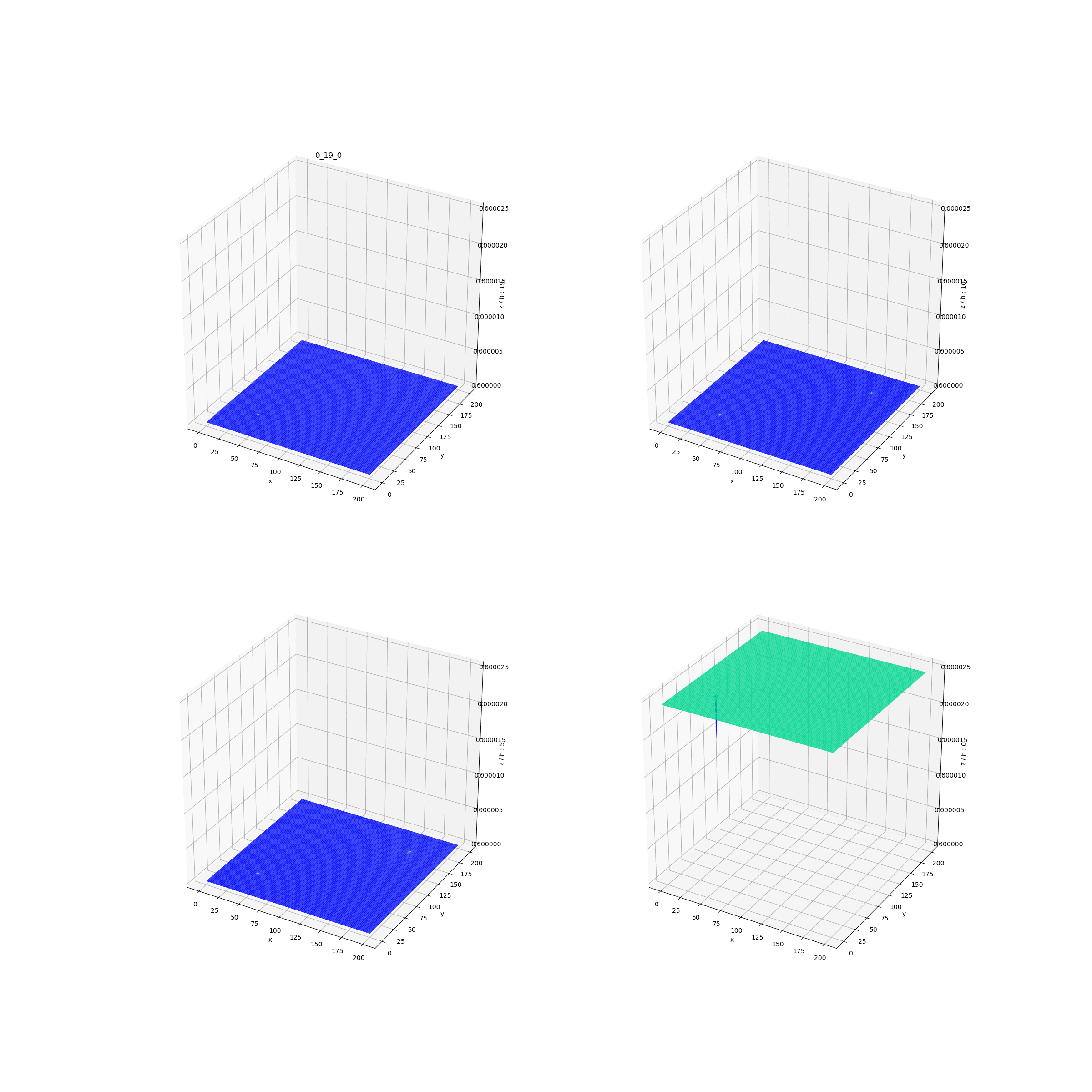}
        \caption{A few more timesteps.}
    \end{subfigure}
    \begin{subfigure}{0.27\textwidth}
        \centering
        \includegraphics[width=1.0\linewidth]{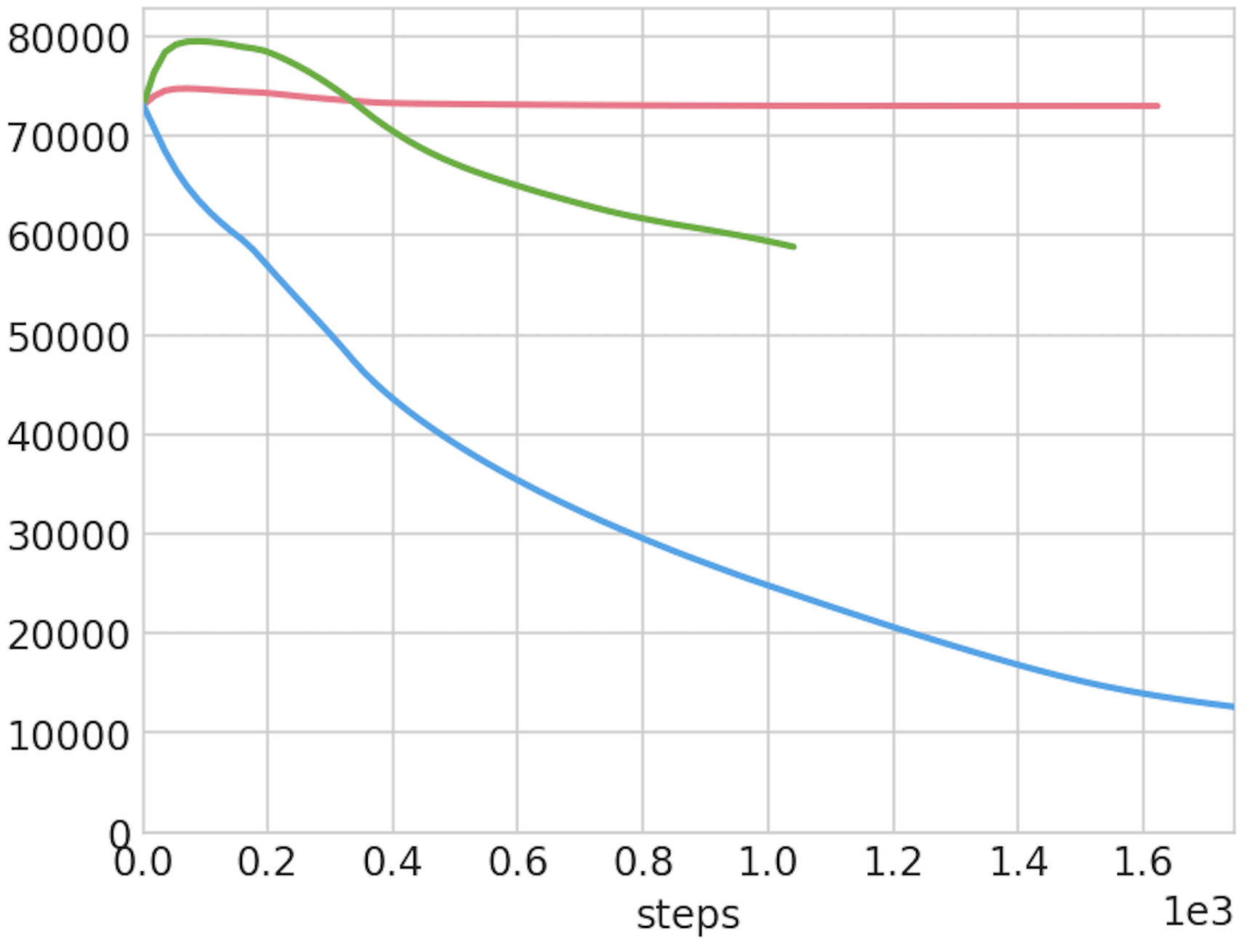}
        \caption{Exploitability of FP, FP damped and OMD.}
        \label{subfig:building_exploitability_omd}
    \end{subfigure}
    \caption{Population distribution at consecutive dates. Each plot of a subfigure is a different floor, the bottom floor is the bottom-right plot, the top floor is the top-left plot. Fig.~\ref{subfig:building_exploitability_omd}: FP (red, $\alpha=10^{-5}$), FP damped (green, $\alpha= 10^{-3}$) and OMD (blue, $\alpha=10^{-4}$).}
    \label{fig:Building}
\end{figure*}

\subsection{Building evacuation}

{\bf Environment: }
We now turn to a single-population crowd modeling problem, namely a building evacuation. This kind of problem has been the topic of several works on MFG (see e.g.~\cite{MR3392611,achdou2019mean} for a single room and~\cite{djehiche2017mean} for a multilevel building). The building consists of 20 floors, each of dimension $200 \times 200$. At each floor, two staircases are located at two opposite corners, such as the crowd has to cross the whole floor to take the next staircase. Each agent can remain in place, move in the 4 directions (up, down, right, left) as well as go up or down when on a staircase location. The initial distribution is uniform over all the floors. Each agent of the crowd wants to go downstairs as quickly as possible - as it gets a reward of $10$ at the bottom floor - while favoring social distancing:
\begin{equation*}
    r(x,a,\mu) = -\eta \log(\mu(x)) + 10 \times \mathds{1}_{floor = 0}
\end{equation*}
\begin{wrapfigure}{r}{0.25\textwidth}
    \centering
      \includegraphics[width=0.3\linewidth]{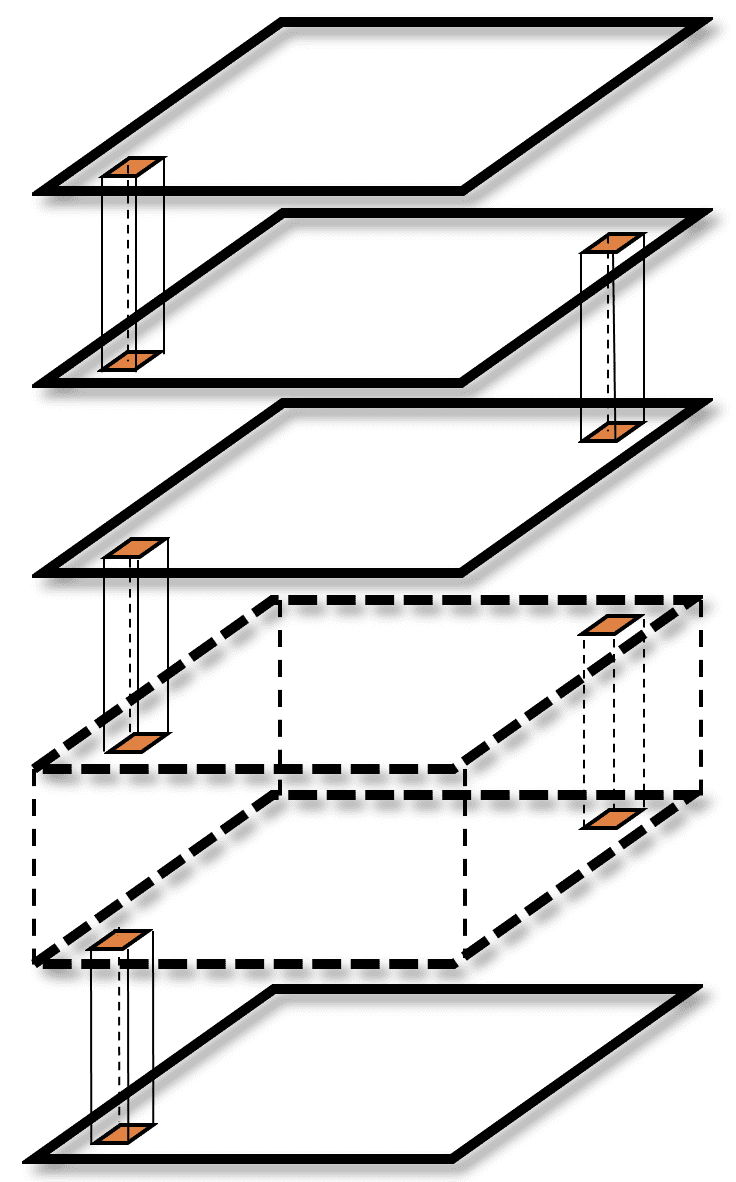}
\caption{Building environment.}
    \label{fig:Building_env}
\end{wrapfigure}
{\bf Numerical results:} We compute this problem with a horizon of $10000$, so $|\mathcal{X}| = 8^{10}$. We take $\eta = 1$. To ensure that the reward stays bounded, we clip the first part $-\eta \log(\mu(x))$ to $-40$. As expected, we observe in Fig.~\ref{fig:Building} that the agents go downstairs and do not concentrate on the shortest path but rather spread mildly. OMD converges faster than both FP and FP damped.

\subsection{Crowd motion with randomly shifted point of interest}
\label{sec:crowd-random-shift}

{\bf Environment:} We consider a second crowd modeling MFG, extending the Beach Bar problem of \cite{perrin2020fictitious} in two dimensions. The environment is a 2D torus of dimensions $1000 \times 1000$, with a point of interest initially located at the center of the square. After $200$ timesteps, the point of interest changes location, moving randomly in the direction of one of the corner. This process repeats itself $5$ times. This random location change adds common noise to the environment and increases exponentially the number of states. Considering MFG with common noise can be encompassed in our previous study by simply increasing the state space with the common noise and adding time to the reward and the transition kernel.  For every random movement, four possible directions are possible, making the total number of states $|\mathcal{X}| = 2\times 10^{8} \times \sum_{k=0}^4 4^k = 2.73\times 10^{11}$ states.
The reward is:  $r(x,a,\mu) = C \times (1 - \frac{\|bar - (i,j)\|_1}{2 \times N_{side}} ) - \log(\mu(x))$.\\
\begin{figure*}[htbp]
    \centering
    \begin{subfigure}{0.23\textwidth}
        \centering
        \includegraphics[width=1.0\linewidth]{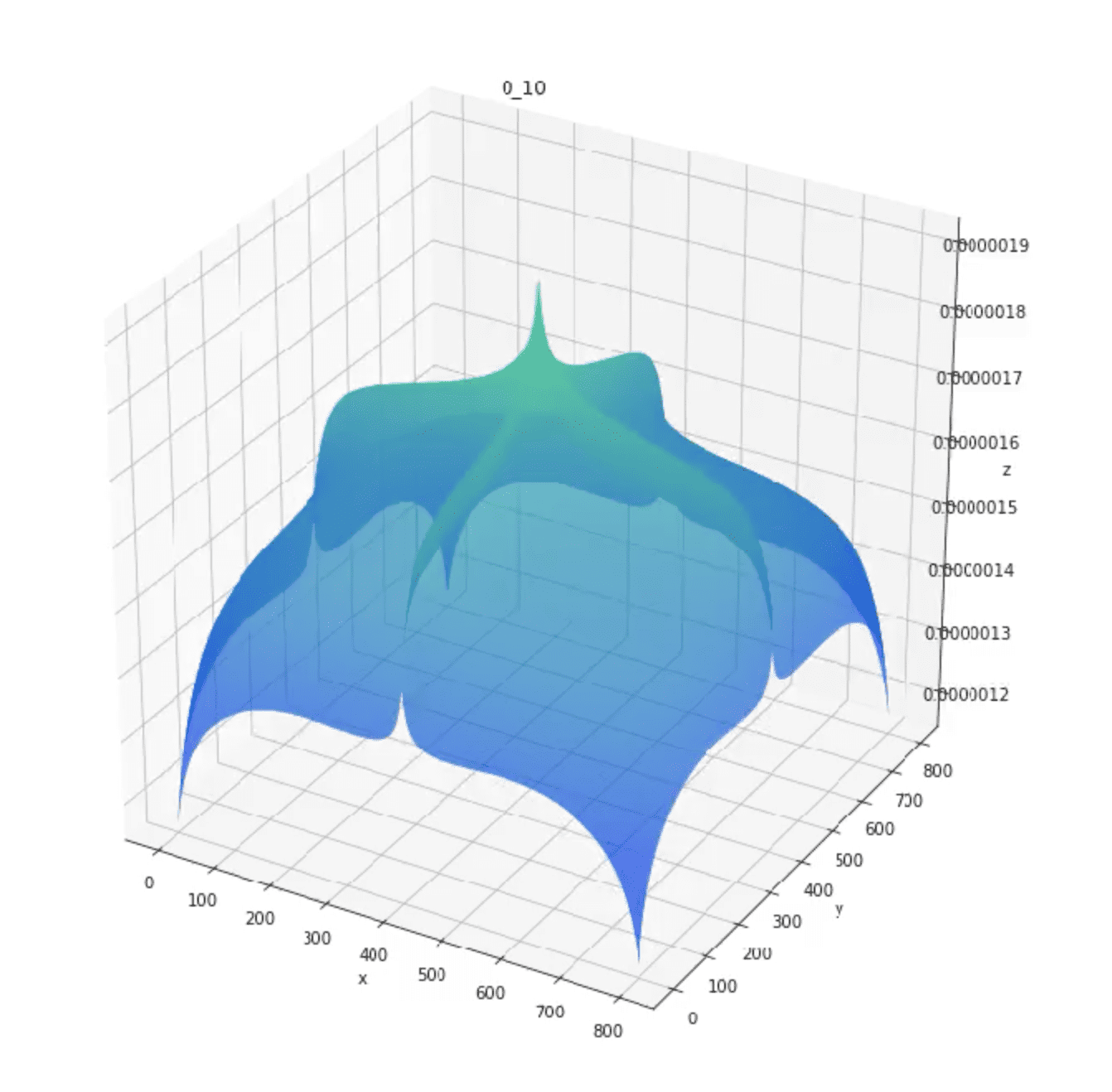}
        \caption{After 90 timesteps.}
    \end{subfigure}
    \begin{subfigure}{0.23\textwidth}
        \centering
        \includegraphics[width=1.0\linewidth]{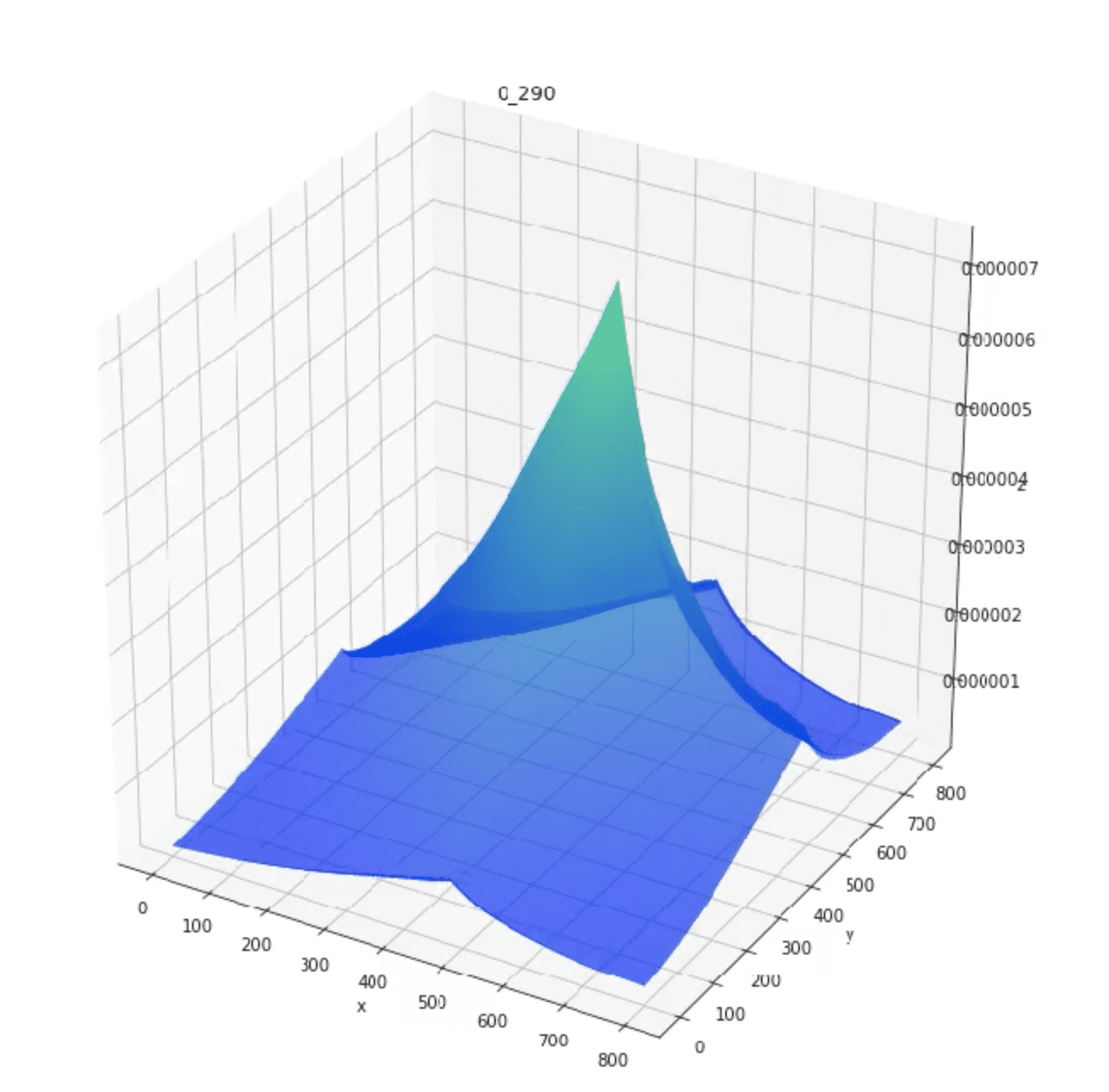}
        \caption{The crowd follows the point.}
    \end{subfigure}
    \begin{subfigure}{0.23\textwidth}
        \centering
        \includegraphics[width=1.0\linewidth]{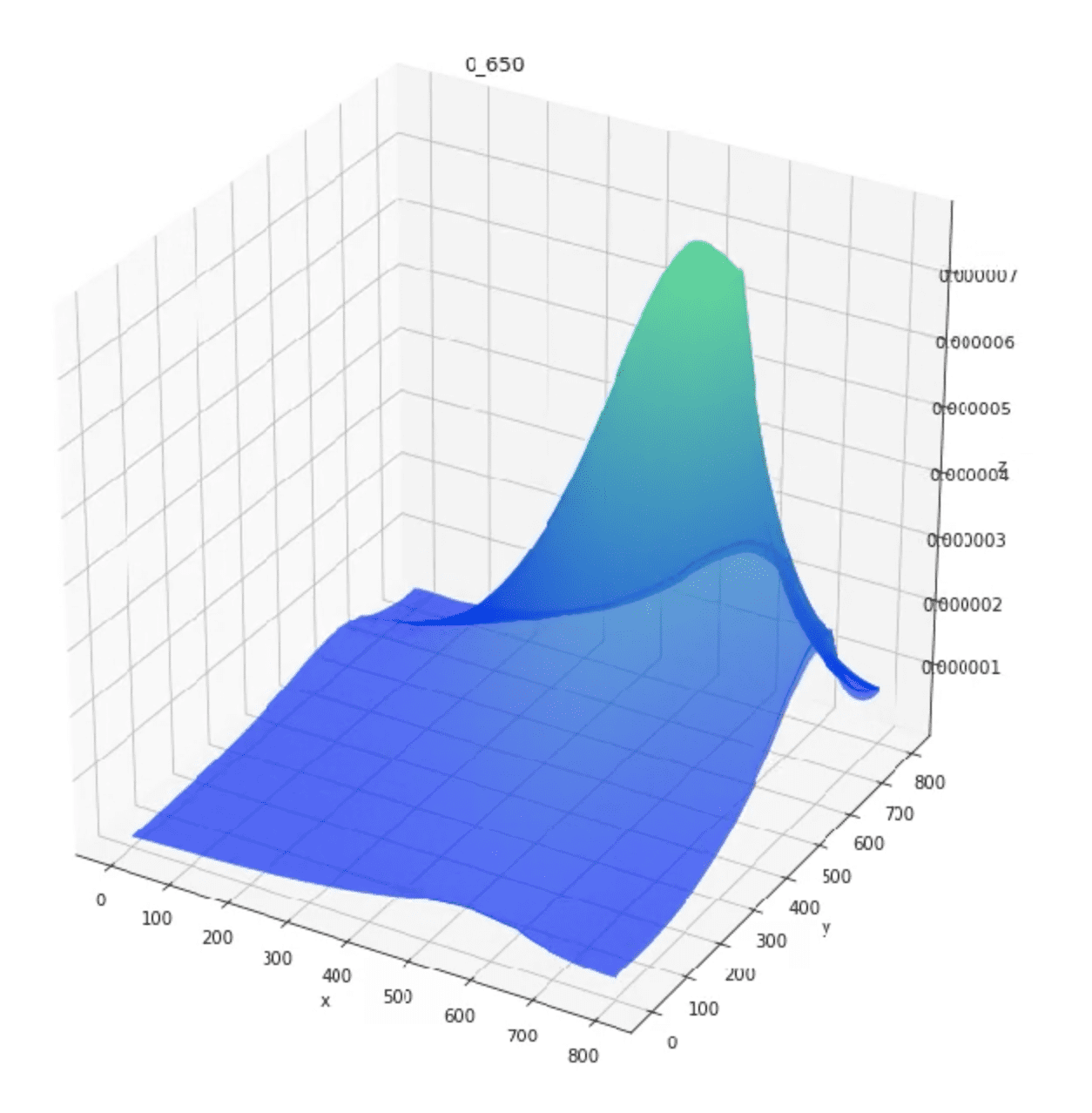}
        \caption{A few more timesteps.}
    \end{subfigure}
    \begin{subfigure}{0.27\textwidth}
        \centering
        \includegraphics[width=1.0\linewidth]{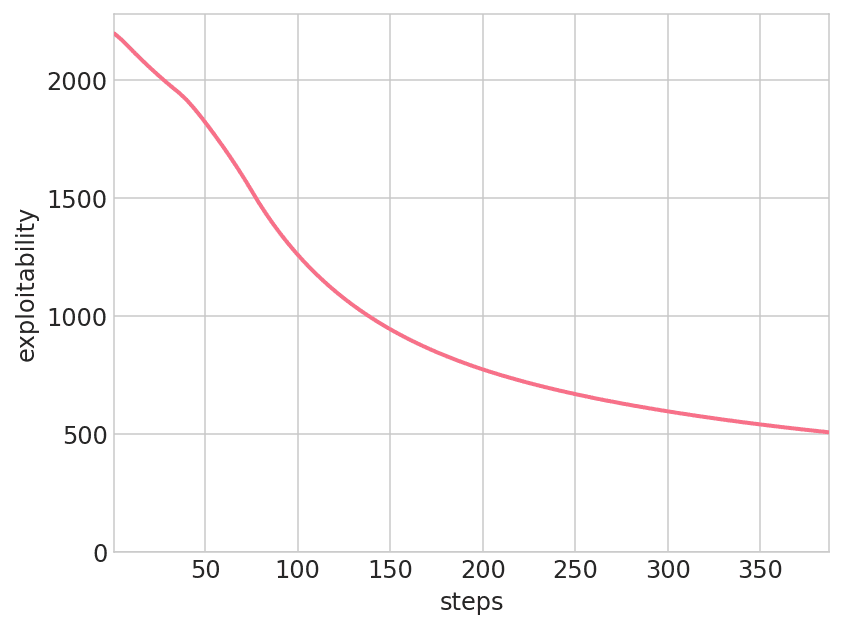}
        \caption{Exploitability of OMD.}
    \end{subfigure}
    \caption{Crowd position at different consecutive dates when the point of interest is randomly shifted to the right by a common noise.}
    \label{fig:Crowd_Common_Noise}
\end{figure*}

{\bf Numerical results:} We set $C = 10$. We observe in Fig.~\ref{fig:Crowd_Common_Noise} that the population is organizing itself with respect to the point of interest and follows it closely as it randomly moves within the dedicated square region. With common noise we get more than a trillion states, making it hard for FP to scale. More plots with a smaller state space are available in Appx.~\ref{app:expe} for a comparison of OMD and FP.
\subsection{Multi-population chasing}

{\bf Environment: } We finally look at \MPMFG{}s, where the populations are chasing each other in a cyclic manner. For the sake of clarity, we explain the reward structure with 3 populations, but more populations are considered in the experiments. With three populations, the game closely relates to the well known Hens-Foxes-Snakes outdoor game for kids. Hens are trying to catch snakes, while snakes are chasing foxes, who are willing to eat hens. It can also be interpreted as a control version of the spatially extended Rock-Paper-Scissors, where patterns of travelling waves appear under certain conditions \cite{Postlethwaite_2017}. The interplay between nontransitive interactions and biodiversity has been the subject of extensive, mostly experimental, research showing that the setting details  critically affect the emergent behavior~\citep{szolnoki2020pattern}.

To ensure $\bar r^{i,j} = - \bar r^{j,i}$ we implement MP-MFGs with the reward structure defined in Table~\ref{table:RPS} (ex. with 3 populations).

\begin{table}[htb]
\centering
\begin{tabular}{l|ccc}
   & R & P & S  \\ \hline      
 R & 0  & -1  & 1   \\ 
 P & 1  &  0  & -1  \\ 
 S & -1  & 1 &  0 
\end{tabular}
\caption{$\bar r^{i,j}$ for three-population.\label{table:RPS}}
\vspace{-5mm}
\end{table}
The reward of population $i$ is monotone (\textit{cf.} Appx.~\ref{appx:subsec_proof_monotone_multi_reward}) and follows the definition \eqref{eq:reward_multi}:
$
    r^i(x, a, \mu^1, \dots, \mu^N) = -\log(\mu^i(x)) + \sum_{j\neq i}\mu^j(x)\bar r^{i,j}(x).
$
The distributions are initialized either randomly or in different corners. The number of agents of each population is fixed, but the reward encourages the agent to chase the population that it dominates. For example, if an agent is Rock, the second term of the reward is proportional to the amount of Scissors agents $\mu^{S}$ where the Rock agent is located, and inversely proportional to $\mu^{P}$, the proportion of Paper agents, making the Rock agent to flee from places populated by Paper agents.

{\bf Numerical results: } We present a four-population example, each is initially located at a corner of the environment. We observe that the populations are chasing each other in a cyclic fashion. Fig.~\ref{fig:multipop_main} highlights that OMD algorithm outperforms FP in terms of exploitability minimization (full comparison with different values of $\alpha$ in Appx.~\ref{appx:sec_multipop_expes}). It demonstrates the robustness of the OMD algorithm within the different topologies considered. Topologies of the environment are a torus, a basic square or the `donut' topology (an environment where the agent gets a negative reward if it goes inside a large zone at the center of the square).

\begin{figure*}[htbp]
    \centering
    \begin{subfigure}{0.23\textwidth}
        \centering
        \includegraphics[width=1.0\linewidth]{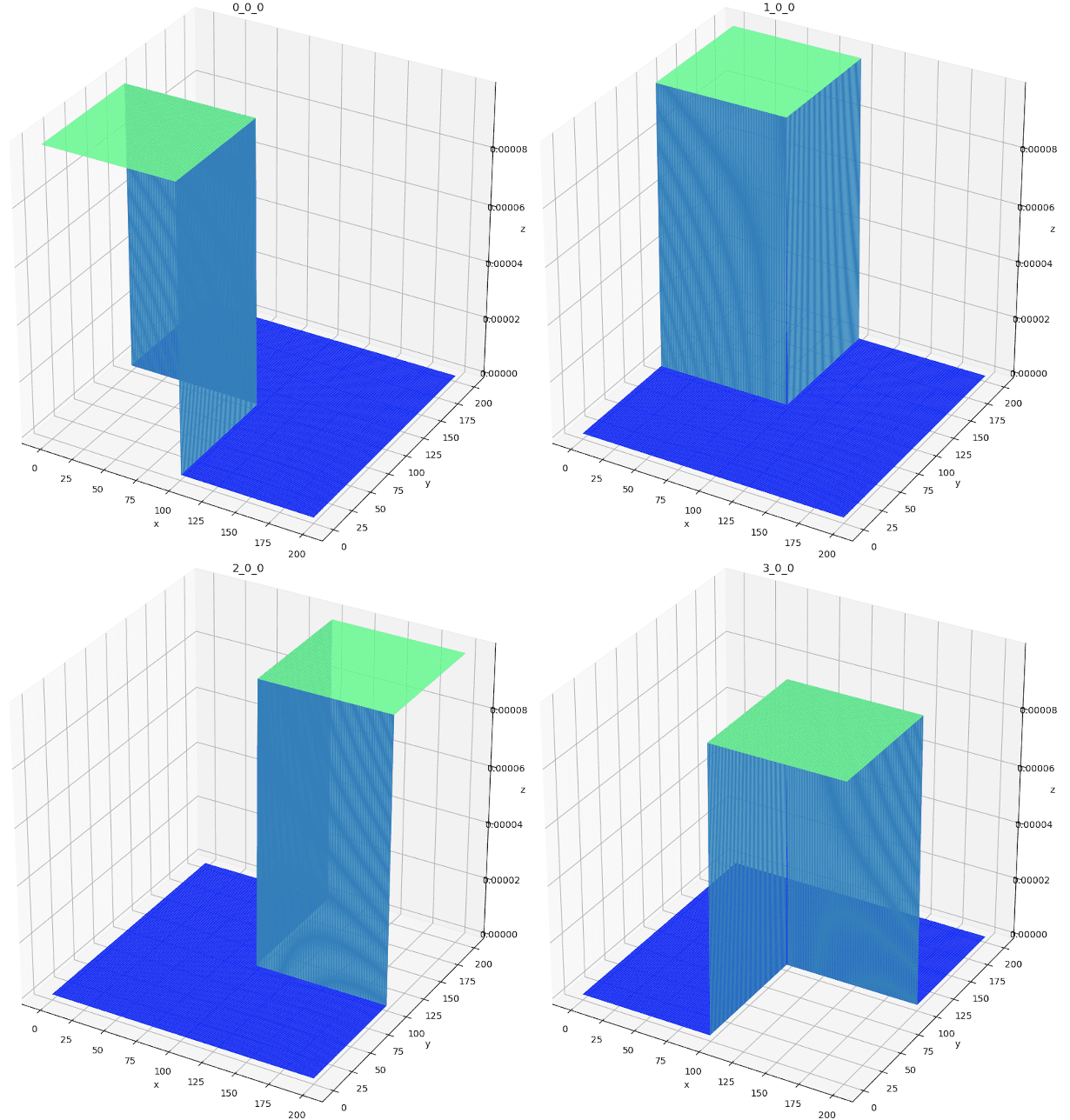}
        \caption{Initial distributions.}
    \end{subfigure}
    \begin{subfigure}{0.23\textwidth}
        \centering
        \includegraphics[width=1.0\linewidth]{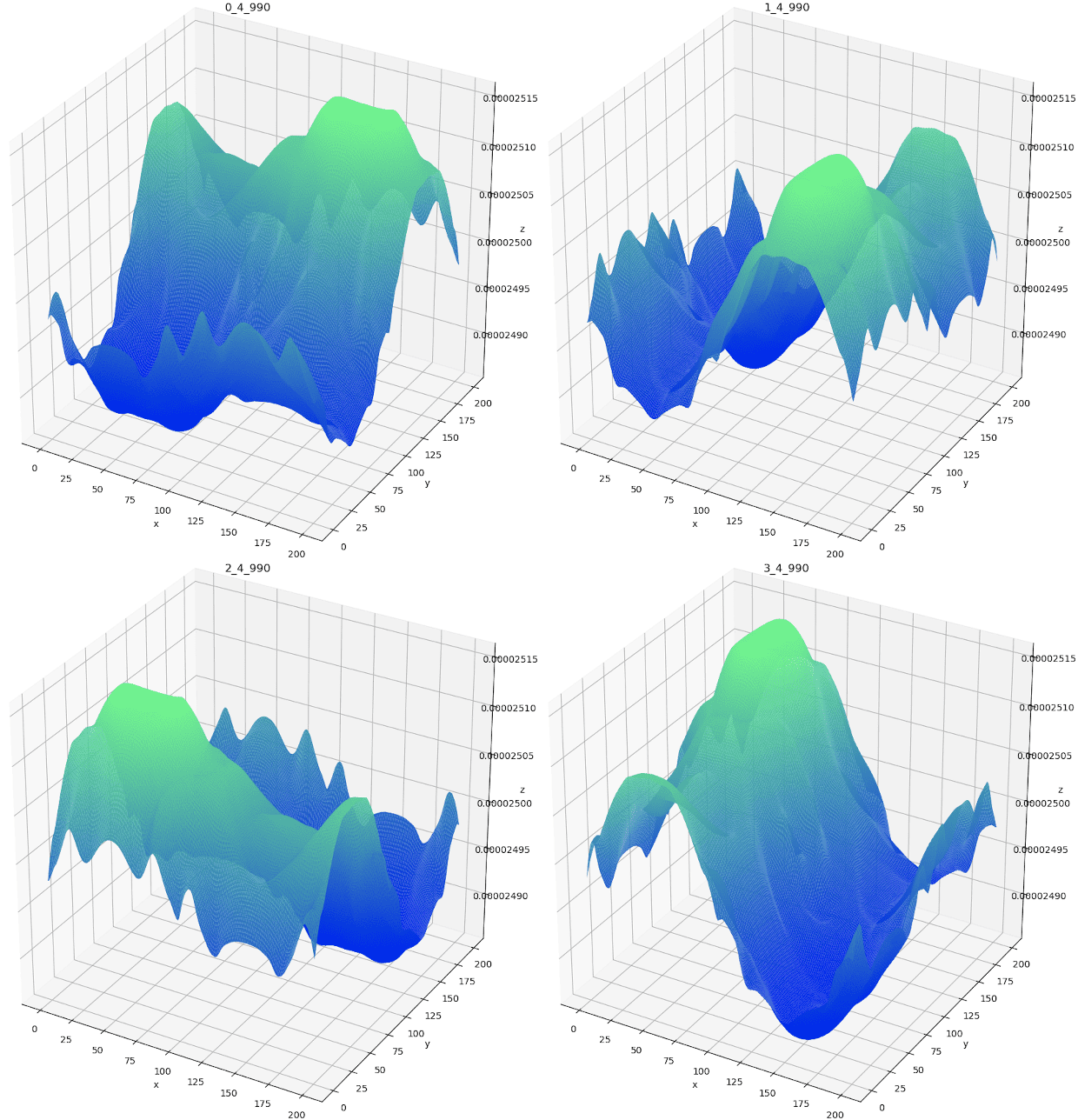}
        \caption{After a few timesteps.}
    \end{subfigure}
    \begin{subfigure}{0.23\textwidth}
        \centering
        \includegraphics[width=1.0\linewidth]{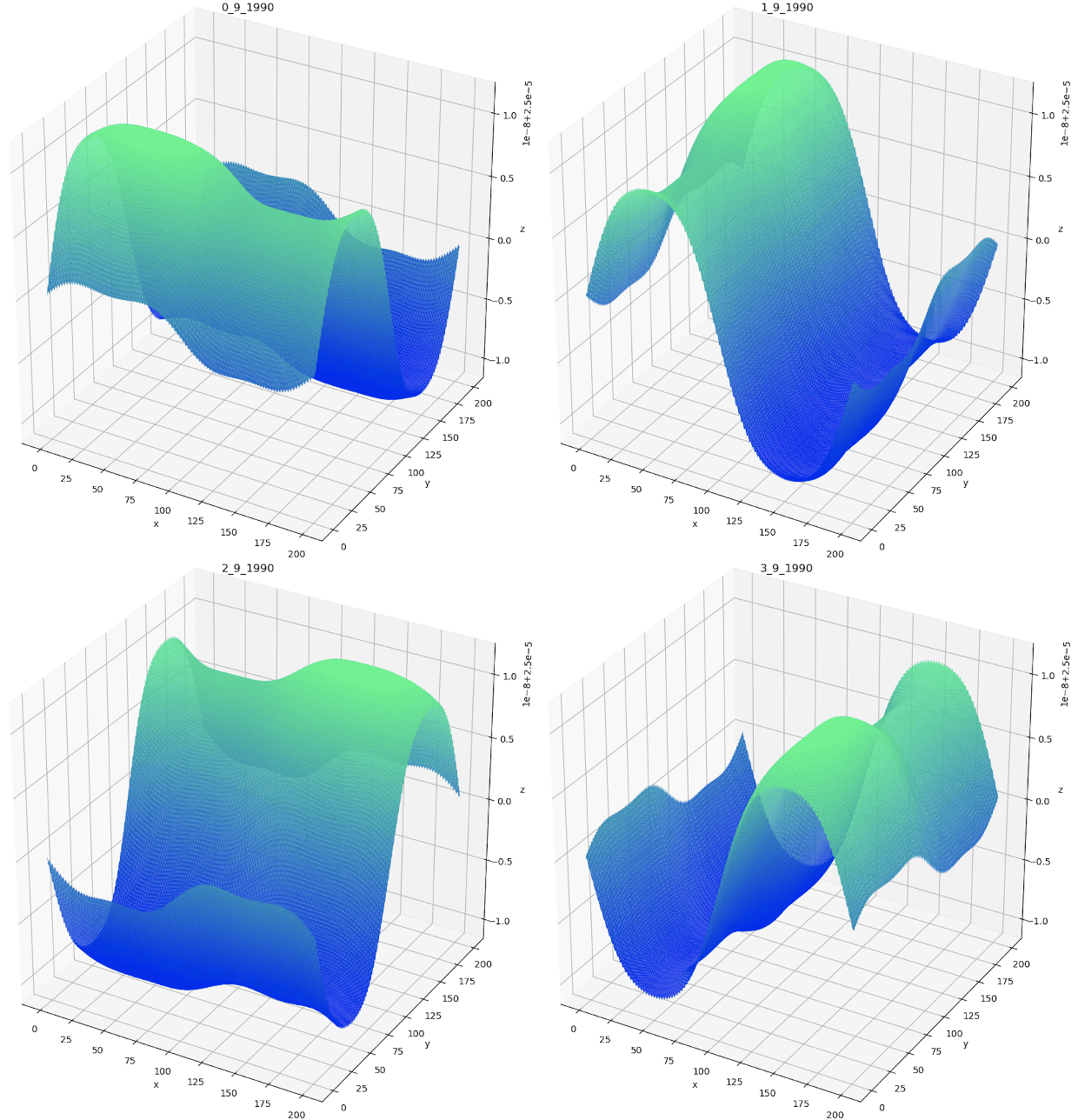}
        \caption{A few more timesteps.}
    \end{subfigure}
    \begin{subfigure}{0.27\textwidth}
        \centering
        \includegraphics[width=1.0\linewidth]{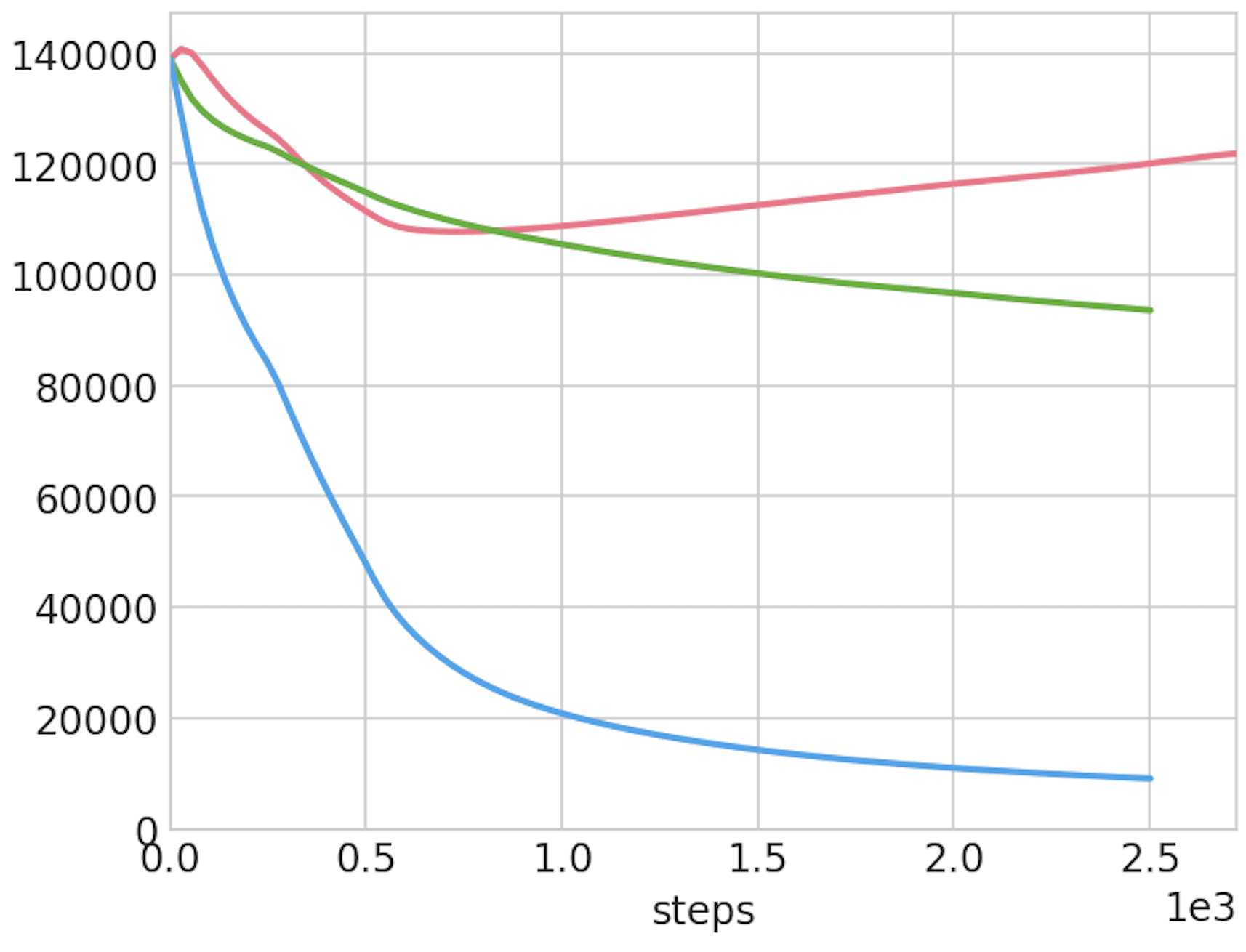}
        \caption{Exploitability.}
        \label{subfig:4pop_exploitability_omd}
    \end{subfigure}
    \caption{4-population chasing. Fig~\ref{subfig:4pop_exploitability_omd}: FP (red, $\alpha=10^{-3}$), FP damped (green, $\alpha = 10^{-5}$) and OMD (blue, $\alpha = 10^{-5}$).}
    \label{fig:multipop_main}
\end{figure*}

\section{Related work}

OMD dynamics have been studied extensively within the field of multi-agent games~\cite{Cesa06,10.5555/1296179}. Leveraging the well known advantageous regret properties of such dynamics~\cite{srebro2011universality}, one can prove strong time-average convergence results both in zero-sum games (and network variants thereof)~\cite{freund1999adaptive,cai2016zero} as well as in smooth-games~\cite{Roughgarden09}.
Recently, there has been explicit focus on understanding their day-to-day behavior which has been shown to be non-equilibrating even in standard bilinear zero-sum games~\cite{piliouras2014optimization,mertikopoulos2017cycles}.
 Moreover, even in simple games the behavior of such dynamics can become formally chaotic~\cite{Sato02042002,palaiopanos2017multiplicative,chotibut2019route}.
 Nevertheless, sufficient conditions have been established under which converge to \Nash is guaranteed even in the sense of the day-to-day behavior~\cite{zhou2017mirror,bravo2018bandit}. 
  We find sufficient conditions for convergence in the more demanding setting of  \MPMFG. 

\MPMFG{}s have been introduced in~\cite{MR2346927} and studied from a PDE viewpoint in~\cite{MR3127148,MR3333058,MR3660463,MR3888965}. To the best of our knowledge, our work is the first one to provide a monotonicity condition for \MPMFG and to provide a provably converging algorithm.

Related to the question of learning in MFGs,~\cite{yin2010learning} studied a MF oscillator game, while~\cite{cardaliaguet2017learning} initiated the study of Fictitious Play in MFGs, which has been further studied in~\cite{hadikhanloo2019finite}. Recently, these ideas have been combined with Reinforcement Learning by~\citet{elie2020convergence,perrin2020fictitious}. These methods allow solving MFGs under a monotonicity assumption, which is at the same time easier to check and less restrictive than the ones used to ensure convergence for fixed point iterations~\citep{guo2019learning,anahtarci2020q} or single-loop fictitious play iterations~\citep{angiuli2020unified,xie2020provable}. In our work, we also prove convergence under such a weak monotonicity condition, which enables us to cover a large class of MFGs. Furthermore, we consider time-dependent problems (as e.g. in~\cite{mishra2020model}) instead of stationary equilibria. Mirror Descent for MFGs has been introduced in~\cite{hadikhanloo2017learning} for first-order, single-population MFG. Our results cover second order, \MPMFG. Traditional numerical methods for solving MFGs typically rely on a finite difference scheme 
introduced in~\cite{MR2679575}. This approach can be extended to solve \MPMFG, see~\cite{MR3597009}. However, to the best of our knowledge, there is no general convergence guarantees, nor  has it been tested on examples with as many states as we consider. More recently, several numerical methods to solve MFGs based on machine learning tools have been proposed using either an analytical viewpoint~\cite{al2018solving,carmona2019convergence-I,ruthotto2020machine,cao2020connecting,lin2020apac} or a stochastic viewpoint~\cite{fouque2019deep,carmona2019convergence-II,germain2019numerical}. To the best of our knowledge, these algorithms have not been proved to converge and seem applicable only under rather stringent conditions (on the structure or the regularity of the problem) and do not seem to be directly applicable to complex geometries due to boundary conditions. Last, the question of learning with multiple infinite populations of agents has also been studied recently in~\cite{subramanian2018reinforcement}. The authors consider several groups where the agents cooperate among each group, which differs from our setting where all the agents compete. 

\section{Conclusion}

We proposed Online Mirror Descent for \MPMFG{}s. We have proved that under appropriate monotonicity assumptions, OMD converges to a \Nash. Moreover, we considered multiple experimental benchmarks, some with hundreds of billions states, and compared extensively OMD to FP. OMD scales up remarkably well and consistently converges significantly faster than FP.  An interesting direction of future work would be to study  the rate of convergence of OMD. We have shown $O(1/t)$ rate for FP in \MPMFG  but our technique does not extend to OMD.
 Empirically, we envision to extend this approach to a model-free setting with function approximation and address even larger problems.

\clearpage

\section*{Acknowledgments}

Georgios Piliouras gratefully acknowledges grant PIE-SGP-AI-2018-01, NRF2019-NRF-ANR095 ALIAS grant and NRF 2018 Fellowship NRF-NRFF2018-07. Mathieu Lauri{\`e}re gratefully acknowledges the support of NSF grant DMS-1716673 and ARO grant  W911NF-17-1-0578. We would like to thanks Mark Rowland for his review of the manuscript and helpful suggestions.

\bibliography{example_paper}

\begin{thebibliography}{86}
\providecommand{\natexlab}[1]{#1}
\providecommand{\url}[1]{\texttt{#1}}
\expandafter\ifx\csname urlstyle\endcsname\relax
  \providecommand{\doi}[1]{doi: #1}\else
  \providecommand{\doi}{doi: \begingroup \urlstyle{rm}\Url}\fi

\bibitem[Achdou \& Capuzzo-Dolcetta(2010)Achdou and
  Capuzzo-Dolcetta]{MR2679575}
Achdou, Y. and Capuzzo-Dolcetta, I.
\newblock Mean field games: numerical methods.
\newblock \emph{SIAM Journal on Numerical Analysis}, 48\penalty0 (3), 2010.
\newblock ISSN 0036-1429.
\newblock \doi{10.1137/090758477}.
\newblock URL \url{https://doi.org/10.1137/090758477}.

\bibitem[Achdou \& Lasry(2019)Achdou and Lasry]{achdou2019mean}
Achdou, Y. and Lasry, J.-M.
\newblock Mean field games for modeling crowd motion.
\newblock In \emph{Contributions to partial differential equations and
  applications}. Springer, 2019.

\bibitem[Achdou \& Lauri\`ere(2015)Achdou and Lauri\`ere]{MR3392611}
Achdou, Y. and Lauri\`ere, M.
\newblock On the system of partial differential equations arising in mean field
  type control.
\newblock \emph{Discrete Contin. Dyn. Syst.}, 35\penalty0 (9), 2015.
\newblock ISSN 1078-0947.
\newblock \doi{10.3934/dcds.2015.35.3879}.
\newblock URL \url{https://doi.org/10.3934/dcds.2015.35.3879}.

\bibitem[Achdou \& Lauri{\`e}re(2020)Achdou and
  Lauri{\`e}re]{achdoulauriere2020mfgnumerical}
Achdou, Y. and Lauri{\`e}re, M.
\newblock Mean field games and applications: Numerical aspects.
\newblock In \emph{Mean Field Games}, volume 2281 of \emph{C.I.M.E. Foundation
  Subseries}. Springer International Publishing, 2020.

\bibitem[Achdou et~al.(2012)Achdou, Camilli, and
  Capuzzo-Dolcetta]{achdou2012mean}
Achdou, Y., Camilli, F., and Capuzzo-Dolcetta, I.
\newblock Mean field games: numerical methods for the planning problem.
\newblock \emph{SIAM Journal on Control and Optimization}, 50\penalty0 (1),
  2012.

\bibitem[Achdou et~al.(2014)Achdou, Buera, Lasry, Lions, and
  Moll]{achdou2014pde}
Achdou, Y., Buera, F., Lasry, J.-M., Lions, P.-L., and Moll, B.
\newblock {PDE} models in macroeconomics.
\newblock \emph{Proceedings of the Royal Society of London. Series A,
  Mathematical and Physical Sciences}, 2014.

\bibitem[Achdou et~al.(2017)Achdou, Bardi, and Cirant]{MR3597009}
Achdou, Y., Bardi, M., and Cirant, M.
\newblock Mean field games models of segregation.
\newblock \emph{Math. Models Methods Appl. Sci.}, 27\penalty0 (1):\penalty0
  75--113, 2017.
\newblock ISSN 0218-2025.
\newblock \doi{10.1142/S0218202517400036}.
\newblock URL \url{http://dx.doi.org/10.1142/S0218202517400036}.

\bibitem[Al-Aradi et~al.(2018)Al-Aradi, Correia, Naiff, Jardim, and
  Saporito]{al2018solving}
Al-Aradi, A., Correia, A., Naiff, D., Jardim, G., and Saporito, Y.
\newblock Solving nonlinear and high-dimensional partial differential equations
  via deep learning.
\newblock \emph{arXiv preprint arXiv:1811.08782}, 2018.

\bibitem[Anahtarci et~al.(2020)Anahtarci, Kariksiz, and Saldi]{anahtarci2020q}
Anahtarci, B., Kariksiz, C.~D., and Saldi, N.
\newblock Q-learning in regularized mean-field games.
\newblock \emph{arXiv preprint arXiv:2003.12151}, 2020.

\bibitem[Angiuli et~al.(2019)Angiuli, Graves, Li, Chassagneux, Delarue, and
  Carmona]{angiuli2019cemracs}
Angiuli, A., Graves, C.~V., Li, H., Chassagneux, J.-F., Delarue, F., and
  Carmona, R.
\newblock Cemracs 2017: numerical probabilistic approach to {MFG}.
\newblock \emph{ESAIM: Proceedings and Surveys}, 65, 2019.

\bibitem[Angiuli et~al.(2020)Angiuli, Fouque, and
  Lauri{\`e}re]{angiuli2020unified}
Angiuli, A., Fouque, J.-P., and Lauri{\`e}re, M.
\newblock Unified reinforcement q-learning for mean field game and control
  problems.
\newblock \emph{arXiv preprint arXiv:2006.13912}, 2020.

\bibitem[Archibald et~al.(1995)Archibald, McKinnon, and
  Thomas]{archibald1995generation}
Archibald, T., McKinnon, K., and Thomas, L.
\newblock On the generation of markov decision processes.
\newblock \emph{Journal of the Operational Research Society}, 46\penalty0
  (3):\penalty0 354--361, 1995.

\bibitem[Aurell \& Djehiche(2018)Aurell and Djehiche]{MR3763083}
Aurell, A. and Djehiche, B.
\newblock Mean-field type modeling of nonlocal crowd aversion in pedestrian
  crowd dynamics.
\newblock \emph{SIAM J. Control Optim.}, 56\penalty0 (1):\penalty0 434--455,
  2018.
\newblock ISSN 0363-0129.
\newblock \doi{10.1137/17M1119196}.
\newblock URL \url{https://doi.org/10.1137/17M1119196}.

\bibitem[Bardi \& Cardaliaguet(2020)Bardi and
  Cardaliaguet]{bardi2020convergence}
Bardi, M. and Cardaliaguet, P.
\newblock Convergence of some mean field games systems to aggregation and
  flocking models.
\newblock \emph{arXiv:2004.04403}, 2020.

\bibitem[Bardi \& Cirant(2018)Bardi and Cirant]{MR3888965}
Bardi, M. and Cirant, M.
\newblock Uniqueness of solutions in mean field games with several populations
  and {N}eumann conditions.
\newblock In \emph{P{DE} models for multi-agent phenomena}, volume~28 of
  \emph{Springer INdAM Ser.}, pp.\  1--20. Springer, Cham, 2018.

\bibitem[Bensoussan et~al.(2013)Bensoussan, Frehse, and Yam]{MR3134900}
Bensoussan, A., Frehse, J., and Yam, S. C.~P.
\newblock \emph{Mean Field Games and Mean Field Type Control Theory}.
\newblock Springer Briefs in Mathematics. Springer, New York, 2013.
\newblock ISBN 978-1-4614-8507-0; 978-1-4614-8508-7.

\bibitem[Bensoussan et~al.(2018)Bensoussan, Huang, and Lauri\`ere]{MR3882529}
Bensoussan, A., Huang, T., and Lauri\`ere, M.
\newblock Mean field control and mean field game models with several
  populations.
\newblock \emph{Minimax Theory Appl.}, 3\penalty0 (2):\penalty0 173--209, 2018.
\newblock ISSN 2199-1413.

\bibitem[Bravo et~al.(2018)Bravo, Leslie, and Mertikopoulos]{bravo2018bandit}
Bravo, M., Leslie, D.~S., and Mertikopoulos, P.
\newblock Bandit learning in concave $ n $-person games.
\newblock \emph{arXiv preprint arXiv:1810.01925}, 2018.

\bibitem[Brice\~{n}o Arias et~al.(2018)Brice\~{n}o Arias, Kalise, and
  Silva]{MR3772008}
Brice\~{n}o Arias, L.~M., Kalise, D., and Silva, F.~J.
\newblock Proximal methods for stationary mean field games with local
  couplings.
\newblock \emph{SIAM Journal on Control and Optimization}, 56\penalty0 (2),
  2018.
\newblock ISSN 0363-0129.
\newblock \doi{10.1137/16M1095615}.
\newblock URL \url{https://doi.org/10.1137/16M1095615}.

\bibitem[Brice\~no Arias et~al.(2019)Brice\~no Arias, Kalise, Kobeissi,
  Lauri\`ere, Mateos~Gonz\'alez, and Silva]{BricenoAriasetalCEMRACS2017}
Brice\~no Arias, L.~M., Kalise, D., Kobeissi, Z., Lauri\`ere, M.,
  Mateos~Gonz\'alez, A., and Silva, F.~J.
\newblock On the implementation of a primal-dual algorithm for second order
  time-dependent mean field games with local couplings.
\newblock \emph{ESAIM: Proceedings}, 65, 2019.
\newblock \doi{10.1051/proc/201965330}.
\newblock URL \url{https://doi.org/10.1051/proc/201965330}.

\bibitem[Brown \& Sandholm(2017)Brown and Sandholm]{Brown17Libratus}
Brown, N. and Sandholm, T.
\newblock Superhuman {AI} for heads-up no-limit poker: {L}ibratus beats top
  professionals.
\newblock \emph{Science}, 360\penalty0 (6385), December 2017.

\bibitem[Brown \& Sandholm(2019)Brown and Sandholm]{Brown19Pluribus}
Brown, N. and Sandholm, T.
\newblock Superhuman {AI} for multiplayer poker.
\newblock \emph{Science}, 365\penalty0 (6456), 2019.
\newblock ISSN 0036-8075.
\newblock \doi{10.1126/science.aay2400}.
\newblock URL \url{https://science.sciencemag.org/content/365/6456/885}.

\bibitem[Cai et~al.(2016)Cai, Candogan, Daskalakis, and
  Papadimitriou]{cai2016zero}
Cai, Y., Candogan, O., Daskalakis, C., and Papadimitriou, C.
\newblock Zero-sum polymatrix games: A generalization of minmax.
\newblock \emph{Mathematics of Operations Research}, 41\penalty0 (2):\penalty0
  648--655, 2016.

\bibitem[Campbell et~al.(2002)Campbell, Hoane~Jr, and Hsu]{campbell2002deep}
Campbell, M., Hoane~Jr, A.~J., and Hsu, F.-h.
\newblock Deep {Blue}.
\newblock \emph{Artificial intelligence}, 134\penalty0 (1-2), 2002.

\bibitem[Cao et~al.(2020)Cao, Guo, and Lauri{\`e}re]{cao2020connecting}
Cao, H., Guo, X., and Lauri{\`e}re, M.
\newblock Connecting {GANs}, {MFGs}, and {OT}.
\newblock \emph{arXiv preprint arXiv:2002.04112}, 2020.

\bibitem[Cardaliaguet(2012)]{cardaliaguet2010notes}
Cardaliaguet, P.
\newblock Notes on mean field games.
\newblock \emph{P.-L. Lions’ Lectures at Coll\`ege de France}, 2012.

\bibitem[Cardaliaguet \& Hadikhanloo(2017)Cardaliaguet and
  Hadikhanloo]{cardaliaguet2017learning}
Cardaliaguet, P. and Hadikhanloo, S.
\newblock Learning in mean field games: the fictitious play.
\newblock \emph{ESAIM: Control, Optimisation and Calculus of Variations},
  23\penalty0 (2), 2017.

\bibitem[Carlini \& Silva(2014)Carlini and Silva]{MR3148086}
Carlini, E. and Silva, F.~J.
\newblock A fully discrete semi-{L}agrangian scheme for a first order mean
  field game problem.
\newblock \emph{SIAM Journal on Numerical Analysis}, 52\penalty0 (1), 2014.
\newblock ISSN 0036-1429.
\newblock \doi{10.1137/120902987}.
\newblock URL \url{https://doi.org/10.1137/120902987}.

\bibitem[Carlini \& Silva(2015)Carlini and Silva]{MR3392626}
Carlini, E. and Silva, F.~J.
\newblock A semi-{L}agrangian scheme for a degenerate second order mean field
  game system.
\newblock \emph{Discrete and Continuous Dynamical Systems}, 35\penalty0 (9),
  2015.
\newblock ISSN 1078-0947.
\newblock \doi{10.3934/dcds.2015.35.4269}.
\newblock URL \url{https://doi.org/10.3934/dcds.2015.35.4269}.

\bibitem[Carmona \& Delarue(2018{\natexlab{a}})Carmona and Delarue]{MR3752669}
Carmona, R. and Delarue, F.
\newblock \emph{Probabilistic theory of mean field games with applications.
  {I}}, volume~83 of \emph{Probability Theory and Stochastic Modelling}.
\newblock Springer, Cham, 2018{\natexlab{a}}.
\newblock ISBN 978-3-319-56437-1; 978-3-319-58920-6.
\newblock Mean field FBSDEs, control, and games.

\bibitem[Carmona \& Delarue(2018{\natexlab{b}})Carmona and
  Delarue]{carmona2018probabilisticI-II}
Carmona, R. and Delarue, F.
\newblock \emph{Probabilistic Theory of Mean Field Games with Applications
  I-II}.
\newblock Springer, 2018{\natexlab{b}}.

\bibitem[Carmona \& Lauri{\`e}re(2019{\natexlab{a}})Carmona and
  Lauri{\`e}re]{carmona2019convergence-I}
Carmona, R. and Lauri{\`e}re, M.
\newblock {Convergence Analysis of Machine Learning Algorithms for the
  Numerical Solution of Mean Field Control and Games: I--The Ergodic Case}.
\newblock \emph{arXiv preprint arXiv:1907.05980}, 2019{\natexlab{a}}.

\bibitem[Carmona \& Lauri{\`e}re(2019{\natexlab{b}})Carmona and
  Lauri{\`e}re]{carmona2019convergence-II}
Carmona, R. and Lauri{\`e}re, M.
\newblock {Convergence Analysis of Machine Learning Algorithms for the
  Numerical Solution of Mean Field Control and Games: II--The Finite Horizon
  Case}.
\newblock \emph{arXiv preprint arXiv:1908.01613}, 2019{\natexlab{b}}.

\bibitem[Cesa-Bianchi \& Lugosi(2006)Cesa-Bianchi and Lugosi]{Cesa06}
Cesa-Bianchi, N. and Lugosi, G.
\newblock \emph{Prediction, Learning, and Games}.
\newblock Cambridge University Press, 2006.

\bibitem[Chassagneux et~al.(2019)Chassagneux, Crisan, Delarue,
  et~al.]{chassagneux2019numerical}
Chassagneux, J.-F., Crisan, D., Delarue, F., et~al.
\newblock Numerical method for fbsdes of mckean--vlasov type.
\newblock \emph{The Annals of Applied Probability}, 29\penalty0 (3), 2019.

\bibitem[Chotibut et~al.(2019)Chotibut, Falniowski, Misiurewicz, and
  Piliouras]{chotibut2019route}
Chotibut, T., Falniowski, F., Misiurewicz, M., and Piliouras, G.
\newblock The route to chaos in routing games: When is price of anarchy too
  optimistic?
\newblock \emph{arXiv preprint arXiv:1906.02486}, 2019.

\bibitem[Cirant(2015)]{MR3333058}
Cirant, M.
\newblock Multi-population mean field games systems with {N}eumann boundary
  conditions.
\newblock \emph{J. Math. Pures Appl. (9)}, 103\penalty0 (5):\penalty0
  1294--1315, 2015.
\newblock ISSN 0021-7824.
\newblock \doi{10.1016/j.matpur.2014.10.013}.
\newblock URL \url{http://dx.doi.org/10.1016/j.matpur.2014.10.013}.

\bibitem[Cirant \& Verzini(2017)Cirant and Verzini]{MR3660463}
Cirant, M. and Verzini, G.
\newblock Bifurcation and segregation in quadratic two-populations mean field
  games systems.
\newblock \emph{ESAIM Control Optim. Calc. Var.}, 23\penalty0 (3):\penalty0
  1145--1177, 2017.
\newblock ISSN 1292-8119.
\newblock \doi{10.1051/cocv/2016028}.
\newblock URL \url{http://dx.doi.org/10.1051/cocv/2016028}.

\bibitem[Conitzer \& Sandholm(2011)Conitzer and Sandholm]{ConitzerS11}
Conitzer, V. and Sandholm, T.
\newblock Expressive markets for donating to charities.
\newblock \emph{Artif. Intell.}, 175\penalty0 (7-8):\penalty0 1251--1271, 2011.

\bibitem[Couillet et~al.(2012)Couillet, Perlaza, Tembine, and
  Debbah]{couillet2012electrical}
Couillet, R., Perlaza, S.~M., Tembine, H., and Debbah, M.
\newblock Electrical vehicles in the smart grid: A mean field game analysis.
\newblock \emph{IEEE Journal on Selected Areas in Communications}, 30\penalty0
  (6), 2012.

\bibitem[Djehiche et~al.(2017)Djehiche, Tcheukam, and
  Tembine]{djehiche2017mean}
Djehiche, B., Tcheukam, A., and Tembine, H.
\newblock A mean-field game of evacuation in multilevel building.
\newblock \emph{IEEE Transactions on Automatic Control}, 62\penalty0 (10),
  2017.

\bibitem[Elie et~al.(2020)Elie, P{\'e}rolat, Lauri{\`e}re, Geist, and
  Pietquin]{elie2020convergence}
Elie, R., P{\'e}rolat, J., Lauri{\`e}re, M., Geist, M., and Pietquin, O.
\newblock On the convergence of model free learning in mean field games.
\newblock \emph{Proceedings of the AAAI Conference on Artificial Intelligence},
  34\penalty0 (05):\penalty0 7143--7150, 2020.

\bibitem[Feleqi(2013)]{MR3127148}
Feleqi, E.
\newblock The derivation of ergodic mean field game equations for several
  populations of players.
\newblock \emph{Dyn. Games Appl.}, 3\penalty0 (4):\penalty0 523--536, 2013.
\newblock ISSN 2153-0785.
\newblock \doi{10.1007/s13235-013-0088-5}.
\newblock URL \url{http://dx.doi.org/10.1007/s13235-013-0088-5}.

\bibitem[Fouque \& Zhang(2020)Fouque and Zhang]{fouque2019deep}
Fouque, J.-P. and Zhang, Z.
\newblock Deep learning methods for mean field control problems with delay.
\newblock \emph{Frontiers in Applied Mathematics and Statistics}, 6, 2020.
\newblock ISSN 2297-4687.
\newblock \doi{10.3389/fams.2020.00011}.
\newblock URL
  \url{https://www.frontiersin.org/article/10.3389/fams.2020.00011}.

\bibitem[Freedman et~al.(2020)Freedman, Borg, Sinnott{-}Armstrong, Dickerson,
  and Conitzer]{FreedmanBSDC20}
Freedman, R., Borg, J.~S., Sinnott{-}Armstrong, W., Dickerson, J.~P., and
  Conitzer, V.
\newblock Adapting a kidney exchange algorithm to align with human values.
\newblock \emph{Artif. Intell.}, 283:\penalty0 103261, 2020.

\bibitem[Freund \& Schapire(1999)Freund and Schapire]{freund1999adaptive}
Freund, Y. and Schapire, R.~E.
\newblock Adaptive game playing using multiplicative weights.
\newblock \emph{Games and Economic Behavior}, 29\penalty0 (1-2):\penalty0
  79--103, 1999.

\bibitem[Germain et~al.(2019)Germain, Mikael, and Warin]{germain2019numerical}
Germain, M., Mikael, J., and Warin, X.
\newblock Numerical resolution of mckean-vlasov fbsdes using neural networks.
\newblock \emph{arXiv preprint arXiv:1909.12678}, 2019.

\bibitem[Goodfellow et~al.(2016)Goodfellow, Bengio, and
  Courville]{Goodfellow-et-al-2016}
Goodfellow, I., Bengio, Y., and Courville, A.
\newblock \emph{Deep Learning}.
\newblock MIT Press, 2016.
\newblock \url{http://www.deeplearningbook.org}.

\bibitem[Guo et~al.(2019)Guo, Hu, Xu, and Zhang]{guo2019learning}
Guo, X., Hu, A., Xu, R., and Zhang, J.
\newblock Learning mean-field games.
\newblock In \emph{Proceedings of NeurIPS}, 2019.

\bibitem[Hadikhanloo(2017)]{hadikhanloo2017learning}
Hadikhanloo, S.
\newblock Learning in anonymous nonatomic games with applications to
  first-order mean field games.
\newblock \emph{arXiv preprint arXiv:1704.00378}, 2017.

\bibitem[Hadikhanloo \& Silva(2019)Hadikhanloo and
  Silva]{hadikhanloo2019finite}
Hadikhanloo, S. and Silva, F.~J.
\newblock Finite mean field games: fictitious play and convergence to a first
  order continuous mean field game.
\newblock \emph{Journal de Math\'{e}matiques Pures et Appliqu\'{e}es (9)}, 132,
  2019.
\newblock ISSN 0021-7824.
\newblock \doi{10.1016/j.matpur.2019.02.006}.
\newblock URL \url{https://doi.org/10.1016/j.matpur.2019.02.006}.

\bibitem[Huang et~al.(2006{\natexlab{a}})Huang, Malham{\'e}, and
  Caines]{MR2346927}
Huang, M., Malham{\'e}, R.~P., and Caines, P.~E.
\newblock Large population stochastic dynamic games: closed-loop
  {M}c{K}ean-{V}lasov systems and the {N}ash certainty equivalence principle.
\newblock \emph{Communications in Information and Systems}, 6\penalty0 (3),
  2006{\natexlab{a}}.
\newblock ISSN 1526-7555.
\newblock URL \url{http://projecteuclid.org/euclid.cis/1183728987}.

\bibitem[Huang et~al.(2006{\natexlab{b}})Huang, Malham{\'e}, and
  Caines]{MR2346927-HuangCainesMalhame-2006-closedLoop}
Huang, M., Malham{\'e}, R.~P., and Caines, P.~E.
\newblock Large population stochastic dynamic games: closed-loop
  {M}c{K}ean-{V}lasov systems and the {N}ash certainty equivalence principle.
\newblock \emph{Communications in Information and Systems}, 6\penalty0 (3),
  2006{\natexlab{b}}.
\newblock ISSN 1526-7555.
\newblock URL \url{http://projecteuclid.org/euclid.cis/1183728987}.

\bibitem[Lachapelle \& Wolfram(2011)Lachapelle and
  Wolfram]{LachapelleWolfram-2011-MFG-congestion-aversion}
Lachapelle, A. and Wolfram, M.-T.
\newblock On a mean field game approach modeling congestion and aversion in
  pedestrian crowds.
\newblock \emph{Transportation research part B: methodological}, 45\penalty0
  (10):\penalty0 1572--1589, 2011.

\bibitem[Lasry \& Lions(2007)Lasry and Lions]{MR2295621}
Lasry, J.-M. and Lions, P.-L.
\newblock Mean field games.
\newblock \emph{Japanese Journal of Mathematics}, 2\penalty0 (1), 2007.
\newblock ISSN 0289-2316.
\newblock \doi{10.1007/s11537-007-0657-8}.
\newblock URL \url{http://dx.doi.org/10.1007/s11537-007-0657-8}.

\bibitem[Lin et~al.(2020)Lin, Fung, Li, Nurbekyan, and Osher]{lin2020apac}
Lin, A.~T., Fung, S.~W., Li, W., Nurbekyan, L., and Osher, S.~J.
\newblock {}apac-net: Alternating the population and agent control via two
  neural networks to solve high-dimensional stochastic mean field games.
\newblock \emph{arXiv preprint arXiv:2002.10113}, 2020.

\bibitem[Mertikopoulos \& Sandholm(2016)Mertikopoulos and
  Sandholm]{mertikopoulos2016learning}
Mertikopoulos, P. and Sandholm, W.~H.
\newblock Learning in games via reinforcement and regularization.
\newblock \emph{Mathematics of Operations Research}, 41\penalty0 (4):\penalty0
  1297--1324, 2016.

\bibitem[Mertikopoulos et~al.(2018)Mertikopoulos, Papadimitriou, and
  Piliouras]{mertikopoulos2017cycles}
Mertikopoulos, P., Papadimitriou, C., and Piliouras, G.
\newblock Cycles in adversarial regularized learning.
\newblock In \emph{Proceedings of the Twenty-Ninth Annual ACM-SIAM Symposium on
  Discrete Algorithms}, pp.\  2703--2717. SIAM, 2018.

\bibitem[Mishra et~al.(2020)Mishra, Vasal, and Vishwanath]{mishra2020model}
Mishra, R.~K., Vasal, D., and Vishwanath, S.
\newblock Model-free reinforcement learning for non-stationary mean field
  games.
\newblock In \emph{2020 59th IEEE Conference on Decision and Control (CDC)},
  pp.\  1032--1037. IEEE, 2020.

\bibitem[Morav{\v{c}}{\'\i}k et~al.(2017)Morav{\v{c}}{\'\i}k, Schmid, Burch,
  Lis{\`y}, Morrill, Bard, Davis, Waugh, Johanson, and
  Bowling]{moravvcik2017deepstack}
Morav{\v{c}}{\'\i}k, M., Schmid, M., Burch, N., Lis{\`y}, V., Morrill, D.,
  Bard, N., Davis, T., Waugh, K., Johanson, M., and Bowling, M.
\newblock Deepstack: Expert-level artificial intelligence in heads-up no-limit
  poker.
\newblock \emph{Science}, 356\penalty0 (6337), 2017.

\bibitem[Nemirovsky \& Yudin(1979)Nemirovsky and Yudin]{nemirovsky1979problem}
Nemirovsky, A. and Yudin, D.
\newblock Problem complexity and optimization method efficiency.
\newblock \emph{M.: Nauka}, 1979.

\bibitem[Nisan et~al.(2007)Nisan, Roughgarden, Tardos, and
  Vazirani]{10.5555/1296179}
Nisan, N., Roughgarden, T., Tardos, E., and Vazirani, V.~V.
\newblock \emph{Algorithmic Game Theory}.
\newblock Cambridge University Press, USA, 2007.
\newblock ISBN 0521872820.

\bibitem[Othman et~al.(2013)Othman, Pennock, Reeves, and Sandholm]{OthmanPRS13}
Othman, A., Pennock, D.~M., Reeves, D.~M., and Sandholm, T.
\newblock A practical liquidity-sensitive automated market maker.
\newblock \emph{{ACM} Trans. Economics and Comput.}, 1\penalty0 (3):\penalty0
  14:1--14:25, 2013.

\bibitem[Palaiopanos et~al.(2017)Palaiopanos, Panageas, and
  Piliouras]{palaiopanos2017multiplicative}
Palaiopanos, G., Panageas, I., and Piliouras, G.
\newblock Multiplicative weights update with constant step-size in congestion
  games: Convergence, limit cycles and chaos.
\newblock In \emph{Advances in Neural Information Processing Systems}, pp.\
  5872--5882, 2017.

\bibitem[Perrin et~al.(2020)Perrin, P{\'e}rolat, Lauri{\`e}re, Geist, Elie, and
  Pietquin]{perrin2020fictitious}
Perrin, S., P{\'e}rolat, J., Lauri{\`e}re, M., Geist, M., Elie, R., and
  Pietquin, O.
\newblock Fictitious play for mean field games: Continuous time analysis and
  applications.
\newblock \emph{Proc. of NeurIPS}, 2020.

\bibitem[Phelps et~al.(2018)Phelps, Ng, Musolesi, and Russell]{Phelps18}
Phelps, S., Ng, W.~L., Musolesi, M., and Russell, Y.~I.
\newblock Precise time-matching in chimpanzee allogrooming does not occur after
  a short delay.
\newblock \emph{PLOS One}, 13\penalty0 (9), 2018.

\bibitem[Piliouras \& Shamma(2014)Piliouras and
  Shamma]{piliouras2014optimization}
Piliouras, G. and Shamma, J.~S.
\newblock Optimization despite chaos: Convex relaxations to complex limit sets
  via poincar{\'e} recurrence.
\newblock In \emph{Proceedings of the twenty-fifth annual ACM-SIAM symposium on
  Discrete algorithms}, pp.\  861--873. SIAM, 2014.

\bibitem[Postlethwaite \& Rucklidge(2017)Postlethwaite and
  Rucklidge]{Postlethwaite_2017}
Postlethwaite, C.~M. and Rucklidge, A.~M.
\newblock Spirals and heteroclinic cycles in a spatially extended
  rock-paper-scissors model of cyclic dominance.
\newblock \emph{EPL (Europhysics Letters)}, 117\penalty0 (4):\penalty0 48006,
  Feb 2017.
\newblock ISSN 1286-4854.
\newblock \doi{10.1209/0295-5075/117/48006}.
\newblock URL \url{http://dx.doi.org/10.1209/0295-5075/117/48006}.

\bibitem[Robinson(1951)]{robinson1951iterative}
Robinson, J.
\newblock An iterative method of solving a game.
\newblock \emph{Annals of mathematics}, 1951.

\bibitem[Roughgarden(2009)]{Roughgarden09}
Roughgarden, T.
\newblock Intrinsic robustness of the price of anarchy.
\newblock In \emph{Proc. of STOC}, pp.\  513--522, 2009.

\bibitem[Ruthotto et~al.(2020)Ruthotto, Osher, Li, Nurbekyan, and
  Fung]{ruthotto2020machine}
Ruthotto, L., Osher, S.~J., Li, W., Nurbekyan, L., and Fung, S.~W.
\newblock A machine learning framework for solving high-dimensional mean field
  game and mean field control problems.
\newblock \emph{Proceedings of the National Academy of Sciences}, 117\penalty0
  (17), 2020.

\bibitem[Sato et~al.(2002)Sato, Akiyama, and Farmer]{Sato02042002}
Sato, Y., Akiyama, E., and Farmer, J.~D.
\newblock Chaos in learning a simple two-person game.
\newblock \emph{Proceedings of the National Academy of Sciences}, 99\penalty0
  (7):\penalty0 4748--4751, 2002.
\newblock \doi{10.1073/pnas.032086299}.
\newblock URL \url{http://www.pnas.org/content/99/7/4748.abstract}.

\bibitem[Shalev-Shwartz et~al.(2011)]{shalev2011online}
Shalev-Shwartz, S. et~al.
\newblock Online learning and online convex optimization.
\newblock \emph{Foundations and trends in Machine Learning}, 4\penalty0
  (2):\penalty0 107--194, 2011.

\bibitem[Shapiro(1958)]{shapiro1958}
Shapiro, H.~N.
\newblock Note on a computation method in the theory of games.
\newblock In \emph{Communications on Pure and Applied Mathematics}, 1958.

\bibitem[Silver et~al.(2016)Silver, Huang, Maddison, Guez, Sifre, Van
  Den~Driessche, Schrittwieser, Antonoglou, Panneershelvam, Lanctot,
  et~al.]{silver2016mastering}
Silver, D., Huang, A., Maddison, C.~J., Guez, A., Sifre, L., Van Den~Driessche,
  G., Schrittwieser, J., Antonoglou, I., Panneershelvam, V., Lanctot, M.,
  et~al.
\newblock Mastering the game of {Go} with deep neural networks and tree search.
\newblock \emph{Nature}, 529\penalty0 (7587), 2016.

\bibitem[Silver et~al.(2017)Silver, Schrittwieser, Simonyan, Antonoglou, Huang,
  Guez, Hubert, Baker, Lai, Bolton, et~al.]{silver2017mastering}
Silver, D., Schrittwieser, J., Simonyan, K., Antonoglou, I., Huang, A., Guez,
  A., Hubert, T., Baker, L., Lai, M., Bolton, A., et~al.
\newblock Mastering the game of {Go} without human knowledge.
\newblock \emph{Nature}, 550\penalty0 (7676), 2017.

\bibitem[Silver et~al.(2018)Silver, Hubert, Schrittwieser, Antonoglou, Lai,
  Guez, Lanctot, Sifre, Kumaran, Graepel, Lillicrap, Simonyan, and
  Hassabis]{Silver18AlphaZero}
Silver, D., Hubert, T., Schrittwieser, J., Antonoglou, I., Lai, M., Guez, A.,
  Lanctot, M., Sifre, L., Kumaran, D., Graepel, T., Lillicrap, T., Simonyan,
  K., and Hassabis, D.
\newblock A general reinforcement learning algorithm that masters chess, shogi,
  and {G}o through self-play.
\newblock \emph{Science}, 632\penalty0 (6419), 2018.

\bibitem[Srebro et~al.(2011)Srebro, Sridharan, and
  Tewari]{srebro2011universality}
Srebro, N., Sridharan, K., and Tewari, A.
\newblock On the universality of online mirror descent.
\newblock \emph{arXiv preprint arXiv:1107.4080}, 2011.

\bibitem[Subramanian et~al.(2018)Subramanian, Seraj, and
  Mahajan]{subramanian2018reinforcement}
Subramanian, J., Seraj, R., and Mahajan, A.
\newblock Reinforcement learning for mean-field teams.
\newblock In \emph{Workshop on Adaptive and Learning Agents at International
  Conference on Autonomous Agents and Multi-Agent Systems.}, 2018.

\bibitem[Sutton \& Barto(2018)Sutton and Barto]{Sutton2018}
Sutton, R.~S. and Barto, A.~G.
\newblock \emph{Reinforcement Learning: An Introduction}.
\newblock The MIT Press, second edition, 2018.

\bibitem[Szolnoki et~al.(2020)Szolnoki, de~Oliveira, and
  Bazeia]{szolnoki2020pattern}
Szolnoki, A., de~Oliveira, B., and Bazeia, D.
\newblock Pattern formations driven by cyclic interactions: A brief review of
  recent developments.
\newblock \emph{EPL (Europhysics Letters)}, 131\penalty0 (6):\penalty0 68001,
  2020.

\bibitem[Vinyals et~al.(2019)Vinyals, Babuschkin, Czarnecki, Mathieu, Dudzik,
  Chung, Choi, Powell, Ewalds, Georgiev, et~al.]{vinyals2019grandmaster}
Vinyals, O., Babuschkin, I., Czarnecki, W.~M., Mathieu, M., Dudzik, A., Chung,
  J., Choi, D.~H., Powell, R., Ewalds, T., Georgiev, P., et~al.
\newblock Grandmaster level in {StarCraft II} using multi-agent reinforcement
  learning.
\newblock \emph{Nature}, 575\penalty0 (7782), 2019.

\bibitem[Xie et~al.(2020)Xie, Yang, Wang, and Minca]{xie2020provable}
Xie, Q., Yang, Z., Wang, Z., and Minca, A.
\newblock Provable fictitious play for general mean-field games.
\newblock \emph{arXiv preprint arXiv:2010.04211}, 2020.

\bibitem[Yin et~al.(2010)Yin, Mehta, Meyn, and Shanbhag]{yin2010learning}
Yin, H., Mehta, P.~G., Meyn, S.~P., and Shanbhag, U.~V.
\newblock Learning in mean-field oscillator games.
\newblock In \emph{49th IEEE Conference on Decision and Control (CDC)}. IEEE,
  2010.

\bibitem[Zhou et~al.(2017)Zhou, Mertikopoulos, Moustakas, Bambos, and
  Glynn]{zhou2017mirror}
Zhou, Z., Mertikopoulos, P., Moustakas, A.~L., Bambos, N., and Glynn, P.
\newblock Mirror descent learning in continuous games.
\newblock In \emph{2017 IEEE 56th Annual Conference on Decision and Control
  (CDC)}, pp.\  5776--5783. IEEE, 2017.

\bibitem[Zinkevich et~al.(2008)Zinkevich, Johanson, Bowling, and
  Piccione]{zinkevich2008regret}
Zinkevich, M., Johanson, M., Bowling, M., and Piccione, C.
\newblock Regret minimization in games with incomplete information.
\newblock In \emph{Proceedings of NeurIPS}, 2008.

\end{thebibliography}
\bibliographystyle{icml2021}

\clearpage
\onecolumn

\appendix

\section{Separability $+$ Monotonicity Imply Weak Monotonicity}
\label{Separable_Monotonicity_imply_WMonotonicity}

\begin{proof}[Proof of Lemma~\ref{lem:sep-mon-wmon}]
Let us assume that the reward is separable $r^i(x^i, a^i, \mu) = \bar r^i(x^i, a^i) + \tilde r^i(x^i, \mu)$ and that it follows the monotonicity condition: $\forall \mu \neq \mu', \;\sum \limits_i \sum \limits_{x \in \mathcal{X}} (\mu^i(x^i)-{\mu'}^i(x^i))(\tilde r^i(x^i, \mu)-\tilde r^i(x^i, \mu'))\leq 0$. Then, we have:
\begin{align*}
    &\sum \limits_{i=1}^{N_p}\big[J^i(\pi, \mu^{\pi}) - J^i(\pi', \mu^{\pi}) - J^i(\pi, \mu^{\pi'}) +  J^i(\pi', \mu^{\pi'})\big]\\
    & = \sum \limits_{i=1}^{N_p}\sum \limits_{n=0}^N \sum\limits_{(x^i, a^i) \in \mathcal{X}\times \mathcal{A}} \big[ \mu^{i, \pi^i}_n(x^i)\pi^i_n(a^i|x^i)r^i(x^i, a^i, \mu^\pi_{n}) - \mu^{i, {\pi'}^i}_n(x^i){\pi'}^i_n(a^i|x^i)r^i(x^i, a^i, \mu^\pi_{n})\\
    &\qquad\qquad\qquad\qquad\qquad\qquad- \mu^{i, \pi^i}_n(x^i)\pi^i_n(a^i|x^i)r^i(x^i, a^i, \mu^{\pi'}_{n}) +\mu^{i, {\pi'}^i}_n(x^i){\pi'}^i_n(a^i|x^i)r^i(x^i, a^i, \mu^{\pi'}_{n})\big]\\
    &=\sum \limits_{i=1}^{N_p}\sum \limits_{n=0}^N \sum\limits_{(x^i, a^i) \in \mathcal{X}\times \mathcal{A}} \big( \mu^{i, \pi^i}_n(x^i)\pi^i_n(a^i|x^i)-\mu^{i, {\pi'}^i}_n(x^i){\pi'}^i_n(a^i|x^i)\big)\big(r^i(x^i, a^i, \mu^\pi_{n}) - r^i(x^i, a^i, \mu^{\pi'}_{n})\big)
    \\
    &=\sum \limits_{i=1}^{N_p}\sum \limits_{n=0}^N \sum\limits_{(x^i, a^i) \in \mathcal{X}\times \mathcal{A}} \big( \mu^{i, \pi^i}_n(x^i)\pi^i_n(a^i|x^i)-\mu^{i, {\pi'}^i}_n(x^i){\pi'}^i_n(a^i|x^i)\big)\big(\tilde r^i(x^i, \mu^\pi_{n}) - \tilde r^i(x^i, \mu^{\pi'}_{n})\big)
    \\
    &=\sum \limits_{i=1}^{N_p}\sum \limits_{n=0}^N \sum\limits_{x^i \in \mathcal{X}} \big( \mu^{i, \pi^i}_n(x^i)-\mu^{i, {\pi'}^i}_n(x^i)\big)\big(\tilde r^i(x^i, \mu^\pi_{n}) - \tilde r^i(x^i, \mu^{\pi'}_{n})\big)\leq 0.
\end{align*}

With a similar proof, we obtain the corresponding property with strict inequality.
\end{proof}

\section{Multi-Population Reward}
\label{multipop_reward}
Let us suppose:
\begin{align*}
    &r^i(x^i, a^i, \mu) 
    = \bar r^i(x^i, a^i) + \underbrace{\hat r^i(x^i, \mu^i) + \sum_{j\neq i}\mu^j(x^i)\hat r^{i,j}(x^i)}_{=\tilde r^i(x^i,\mu)}
\label{eq:reward_multi}
\end{align*}
With $\forall x\in\mathcal{X}, \hat{r}^{i,j}(x) = - \hat{r}^{j, i}(x)$ and if $\forall \mu \neq \mu', \forall i, \sum \limits_{x \in \mathcal{X}} \Big(\mu^i(x^i)-{\mu'}^i(x^i)\Big)\Big(\hat r^i(x^i, \mu^i)-\hat r^i(x^i, {\mu'}^i)\Big)\leq 0$.

\begin{align*}
    &\sum \limits_i \sum \limits_{x \in \mathcal{X}} (\mu^i(x^i)-{\mu'}^i(x^i))(\tilde r^i(x^i, \mu)-\tilde r^i(x^i, \mu'))\\
    &=\sum \limits_i \sum \limits_{x \in \mathcal{X}} (\mu^i(x^i)-{\mu'}^i(x^i))(\hat r^i(x^i, \mu^i) + \sum_{j\neq i}\mu^j(x^i)\hat r^{i,j}(x^i)-\hat r^i(x^i, {\mu'}^i) - \sum_{j\neq i}{\mu'}^j(x^i)\hat r^{i,j}(x^i))\\
    &=\sum \limits_i \underbrace{\sum \limits_{x \in \mathcal{X}} (\mu^i(x^i)-{\mu'}^i(x^i))(\hat r^i(x^i, \mu^i)-\hat r^i(x^i, {\mu'}^i))}_{\leq 0} + \underbrace{\sum \limits_i \sum_{j\neq i} \sum \limits_{x \in \mathcal{X}} (\mu^i(x^i)-{\mu'}^i(x^i))(\mu^j(x^i)-{\mu'}^j(x^i))\hat r^{i,j}(x^i)}_{=0 \textrm{ since } \hat{r}^{i,j}(x) = - \hat{r}^{j, i}(x)}\\
    &\leq 0
\end{align*}

\newpage
\section{Fictitious Play}
\label{fp_proof}
In this section, we prove that under the weak monotonicity condition, the Fictitious Play process converges to a \Nash.

First, we prove the following property, which stems from the weak monotonicity.

\begin{property}
\label{prop:mono-r-dmu}
Let $f$ be a smooth enough function and let assume that the ODE $\dot \rho = f(\rho)$ (with $\dot \rho = \frac{d}{dt} \rho$) has a solution $(\rho^t)_{t \ge 0} = (\rho^t_n(x))_{t \ge 0, x \in \mathcal{X}}$.
If the game is weakly monotone, then: 
$$\sum \limits_{i=1}^{N_p} \sum \limits_{x^i, a^i \in \mathcal{X}\times \mathcal{A}} \langle \nabla_{\rho} r^i(x^i, a^i,\rho), \dot \rho \rangle \dot \rho^i(x^i, a^i) \leq 0.
$$
\end{property}

\begin{proof}
The monotonicity condition implies that, for all $\tau \geq 0$, we have:
$$\sum \limits_{i=1}^{N_p} \sum \limits_{x^i, a^i \in \mathcal{X}\times \mathcal{A}} (\rho^i_t(x^i, a^i) - \rho^i_{t+\tau}(x^i, a^i))(r^i(x^i, a^i,\rho_t) - r^i(x^i, a^i,\rho_{t+\tau}))\leq 0.
$$
Thus:
$$\sum \limits_{i=1}^{N_p} \sum \limits_{x^i, a^i \in \mathcal{X}\times \mathcal{A}} \frac{\rho^i_t(x^i, a^i) - \rho^i_{t+\tau}(x^i, a^i)}{\tau}\frac{r^i(x^i, a^i,\rho_t) - r^i(x^i, a^i,\rho_{t+\tau})}{\tau}\leq 0.
$$
The result follows when $\tau \rightarrow 0$.
\end{proof}
In the space of distributions over state actions, the Fictitious Play process can be expressed as follows. First, we start with a distribution $\rho^i_{n,t}$ following the balance equation on the state action distributions:
$$\sum \limits_{{a'}^i\in\mathcal{A}} \rho^i_{n-1,t}({x'}^i, {a'}^i) = \sum \limits_{x^i, a^i\in \mathcal{X}\times\mathcal{A}}p({x'}^i|x^i, a^i) \rho^i_{n,t}(x^i, a^i).$$

And for $t<1$, the policy $\pi^i_{n, t}(a^i|x^i) = \frac{\rho^i_{n,t}(x^i, a^i)}{\sum \limits_{a^i\in\mathcal{A}} \rho^i_{n,t}(x^i, a^i)}$ is the uniform policy whenever $\sum \limits_{a^i\in\mathcal{A}} \rho^i_{n,t}(x^i, a^i)>0$.

A best response state action distribution to $\rho$ is written $\rho^{i,br}_{n,t}(x^i, a^i)$ (which will be assumed to be equal to $\rho_t$ for $t<1$) and finally the FP process on the state action distribution is written as for all $t\geq 1$:

$$\rho^i_{n,t}(x^i, a^i) = \frac{1}{t}\int \limits_{0}^{t} \rho^{i,br}_{n,s}(x^i, a^i) ds.$$

The exploitability can then be written as:
$$\phi(t) = \max \limits_{\rho} \Big[\sum \limits_{i=1}^{N_p} \sum_{n=0}^N \sum \limits_{x^i, a^i \in \mathcal{X}\times \mathcal{A}} \rho^{i}_{n}(x^i, a^i) r^i(x^i, a^i,\rho_{n,t})\Big]-\Big[\sum \limits_{i=1}^{N_p} \sum_{n=0}^N \sum \limits_{x^i, a^i \in \mathcal{X}\times \mathcal{A}} \rho^{i}_{n,t}(x^i, a^i) r^i(x^i, a^i,\rho_{n,t})\Big]$$

\begin{property}
    We have that $\frac{d}{dt}\rho^i_{n,t}(x^i, a^i) = \frac{1}{t}\Big[ \rho^{i,br}_{n,s}(x^i, a^i) - \rho^i_{n,t}(x^i, a^i) \Big]$ by taking the derivative of $\rho^i_{n,t}(x^i, a^i) = \frac{1}{t}\int \limits_{0}^{t} \rho^{i,br}_{n,s}(x^i, a^i) ds$ on both sides.
\end{property}

Finally, we take the derivative of the exploitability and get:

\begin{align*}
    \frac{d}{dt} \phi(t) &= \frac{d}{dt}\max \limits_{\rho} \Big[\sum \limits_{i=1}^{N_p} \sum_{n=0}^N \sum \limits_{x^i, a^i \in \mathcal{X}\times \mathcal{A}} \rho^{i}_{n}(x^i, a^i) r^i(x^i, a^i,\rho_{n,t})\Big]-\frac{d}{dt}\Big[\sum \limits_{i=1}^{N_p} \sum_{n=0}^N \sum \limits_{x^i, a^i \in \mathcal{X}\times \mathcal{A}} \rho^{i}_{n,t}(x^i, a^i) r^i(x^i, a^i,\rho_{n,t})\Big]\\
    &=\Big[\sum \limits_{i=1}^{N_p} \sum_{n=0}^N \sum \limits_{x^i, a^i \in \mathcal{X}\times \mathcal{A}} \rho^{i,br}_{n,t}(x^i, a^i) \frac{d}{dt} \Big(r^i(x^i, a^i,\rho_{n,t})\Big)\Big]\\
    &\qquad\qquad-\Big[\sum \limits_{i=1}^{N_p} \sum_{n=0}^N \sum \limits_{x^i, a^i \in \mathcal{X}\times \mathcal{A}} \Big(\rho^{i}_{n,t}(x^i, a^i) \frac{d}{dt} \Big(r^i(x^i, a^i,\rho_{n,t})\Big) + r^i(x^i, a^i,\rho_{n,t}) \frac{d}{dt} \Big(\rho^{i}_{n,t}(x^i, a^i)\Big)\Big)\Big]\\
    &=\Big[\sum \limits_{i=1}^{N_p} \sum_{n=0}^N \sum \limits_{x^i, a^i \in \mathcal{X}\times \mathcal{A}} \underbrace{[\rho^{i,br}_{n,t}(x^i, a^i)-\rho^{i}_{n,t}(x^i, a^i)]}_{=t\frac{d}{dt} \Big(\rho^{i}_{n,t}(x^i, a^i)\Big)} \underbrace{\frac{d}{dt} \Big(r^i(x^i, a^i,\rho_{n,t})\Big)}_{=\langle \nabla_{\rho} r^i(x^i, a^i,\rho_{n,t}), \dot \rho_{n,t} \rangle}\Big]\\
    &\qquad\qquad\qquad\qquad-\Big[\sum \limits_{i=1}^{N_p} \sum_{n=0}^N \sum \limits_{x^i, a^i \in \mathcal{X}\times \mathcal{A}} r^i(x^i, a^i,\rho_{n,t}) \underbrace{\frac{d}{dt} \Big(\rho^{i}_{n,t}(x^i, a^i)\Big)}_{=\frac{1}{t}\Big[ \rho^{i,br}_{n,t}(x^i, a^i) - \rho^i_{n,t}(x^i, a^i) \Big]}\Big]\\
    &=t \sum \limits_{i=1}^{N_p} \sum_{n=0}^N \sum \limits_{x^i, a^i \in \mathcal{X}\times \mathcal{A}} \underbrace{\Big[\frac{d}{dt} \Big(\rho^{i}_{n,t}(x^i, a^i)\Big) \langle \nabla_{\rho} r^i(x^i, a^i,\rho_{n,t}), \dot \rho_{n,t} \rangle \Big]}_{\leq 0}\\
    &\qquad\qquad\qquad\qquad-\frac{1}{t}\underbrace{\Big[\sum \limits_{i=1}^{N_p} \sum_{n=0}^N \sum \limits_{x^i, a^i \in \mathcal{X}\times \mathcal{A}}\Big[ \rho^{i,br}_{n,t}(x^i, a^i) - \rho^i_{n,t}(x^i, a^i) \Big]r^i(x^i, a^i,\rho_{n,t})\Big]}_{=\phi(t)}\\
    &\leq -\frac{1}{t} \phi(t).
\end{align*}

\newpage 
\section{Online Mirror Descent Dynamics}
\label{sec:app-OMD}

\begin{proof}[Proof of Lemma~ \ref{Lemma_Similarity}]

The Continuous Time Online Mirror Descent (CTOMD) algorithm is defined as: for all $t > 0, i\in \{1,\dots,N_p\}, n\in \{0,\dots,N\}$,
\begin{eqnarray*}
    &y^i_{n,t}(x^i,a^i) = \int \limits_{0}^{t} Q^{i,\pi^i_s, \mu^{\pi_s}}_n(x^i,a^i) ds,
    \\
    &\pi^i_{n,t}(.|x^i) = \Gamma(y^i_{n,t}(x^i,.)).
\end{eqnarray*}

\begin{align*}
    \frac{d}{dt}H(y_t) &= \frac{d}{dt} \sum_{i=1}^{N_p} \sum \limits_{n=0}^N \sum \limits_{x^i \in \mathcal{X}} \mu^{i,\pi^*}_n(x^i)\Big[h^*(y^i_{n,t}(x^i,.)) - h^*(y^{i,*}(x^i,.)) - \langle \pi^{i,*}_{n,t},y^i_{n,t}(x^i,.)-y^{i,*}_{n,t}(x^i,.) \rangle \Big]\\
    &=\sum_{i=1}^{N_p} \sum \limits_{n=0}^N \sum \limits_{x^i \in \mathcal{X}} \mu^{i,\pi^*}_n(x^i)\frac{d}{dt} \Big[h^*(y^i_{n,t}(x^i,.)) - h^*(y^{i,*}(x^i,.)) - \langle \pi^{i,*}_{n,t},y^i_{n,t}(x^i,.)-y^{i,*}_{n,t}(x^i,.) \rangle \Big]\\
    &=\sum_{i=1}^{N_p} \sum \limits_{n=0}^N \sum \limits_{x^i \in \mathcal{X}} \mu^{i,\pi^*}_n(x^i) \Big[\frac{d}{dt}h^*(y^i_{n,t}(x^i,.)) - \langle \pi^{i,*}_{n,t},\frac{d}{dt}y^i_{n,t}(x^i,.) \rangle \Big]\\
    &=\sum_{i=1}^{N_p} \sum \limits_{n=0}^N \sum \limits_{x^i \in \mathcal{X}} \mu^{i,\pi^*}_n(x^i) \Big[ \langle \pi^i_{n,t}(.|x^i) - \pi^{i,*}_{n,t}(.|x^i),Q^{i,\pi^i_t, \mu^{\pi_t}}_n(x^i,.) \rangle \Big]\\
    &=\sum_{i=1}^{N_p} \sum \limits_{n=0}^N \sum \limits_{x^i \in \mathcal{X}} \mu^{i,\pi^*}_n(x^i) \Big[V^{i,\pi^i_t, \mu^{\pi_t}}_n(x^i) - \langle \pi^{i,*}_{n,t}(.|x^i),Q^{i,\pi^i_t, \mu^{\pi_t}}_n(x^i,.)\rangle \Big]\\
    &=\sum_{i=1}^{N_p} \sum \limits_{n=0}^N \sum \limits_{x^i \in \mathcal{X}} \mu^{i,\pi^*}_n(x^i) \Big[V^{i,\pi^i_t, \mu^{\pi_t}}_n(x^i) - \langle \pi^{i,*}_{n,t}(.|x^i),r^i(x^i, ., \mu^{\pi_t}) + \sum \limits_{{x'}^i \in \mathcal{X}} p({x'}^i|x^i, a^i)V^{i,\pi^i_t, \mu^{\pi_t}}_{n+1}({x'}^i,.)\rangle \Big]\\
    &=\sum_{i=1}^{N_p} \sum \limits_{n=0}^N \Big[\sum \limits_{x^i \in \mathcal{X}} \mu^{i,\pi^*}_n(x^i) V^{i,\pi^i_t, \mu^{\pi_t}}_n(x^i)\Big] - \Big[\sum\limits_{x^i \in \mathcal{X}} \mu^{i,\pi^*}_n(x^i) \langle \pi^{i,*}_{n,t}(.|x^i),r^i(x^i, ., \mu^{\pi_t}) \Big]\\
    &\qquad\qquad\qquad\qquad- \Big[\sum \limits_{{x'}^i \in \mathcal{X}} V^{i,\pi^i_t, \mu^{\pi_t}}_{n+1}({x'}^i) \underbrace{\sum \limits_{x^i,a^i \in \mathcal{X}\times\mathcal{A}} \mu^{i,\pi^*}_n(x^i) \pi^{i,*}_{n,t}(a^i|x^i) p({x'}^i|x^i, a^i)}_{=\mu^{i,\pi^*}_{n+1}({x'}^i)}\Big]\\
    &=\sum_{i=1}^{N_p} \underbrace{\sum \limits_{n=0}^N \Big[\sum \limits_{x^i \in \mathcal{X}} \mu^{i,\pi^*}_n(x^i) V^{i,\pi^i_t, \mu^{\pi_t}}_n(x^i)\Big] - \sum \limits_{n=0}^N \Big[\sum \limits_{{x'}^i \in \mathcal{X}} V^{i,\pi^i_t, \mu^{\pi_t}}_{n+1}({x'}^i) \mu^{i,\pi^*}_{n+1}({x'}^i)\Big]}_{=J^i(\pi^i_t, \mu^{\pi_t})}\\
    &\qquad\qquad\qquad\qquad - \sum_{i=1}^{N_p} \underbrace{\sum \limits_{n=0}^N \Big[\sum\limits_{x^i \in \mathcal{X}} \mu^{i,\pi^*}_n(x^i) \langle \pi^{i,*}_{n,t}(.|x^i),r^i(x^i, ., \mu^{\pi_t}) \Big]}_{=J^i(\pi^{i,*}, \mu^{\pi_t})}\\
    &=\sum_{i=1}^{N_p}\Big[J^i(\pi^i_t, \mu^{\pi_t})-J^i(\pi^{i,*}, \mu^{\pi_t})\Big] = \Delta J(\pi_t, \pi^*)  + \tilde {\mathcal{M}}(\pi_t, \pi^*) .
\end{align*}
\end{proof}

\newpage 

\section{Weak monotonicity implies $\tilde{ \mathcal{M}} \le 0$}
\label{app:wmon-tildeM}

\begin{proof}[Proof of Lemma~\ref{lem:wmon-tildeM}]
Consider two policies $\pi,\pi'$. Denote by $\mu = \mu^{\pi}, \mu' = \mu^{\pi'}$ respectively the induced distribution sequences. Let $\rho,\rho'$ be the associated joint distribution sequences:
$$
    \rho^i_n(x^i,a^i) = \mu^i_n(x^i) \pi^i_n(a^i|x^i)
$$
and likewise for $\rho'$. By the weak monotonicity, we have: 
\begin{equation}
\label{eq:proof-uniqueness-ineq}
    0 
    \ge \sum \limits_i \sum \limits_{(x^i, a^i) \in \mathcal{X}\times \mathcal{A}} (\rho^i_n(x^i, a^i)-{\rho'}^i_n(x^i, a^i))(r^i(x^i, a^i, \mu_n)-r^i(x^i, a^i, {\mu'}_n))
    =
    \Delta J(\pi,\pi') + \Delta J(\pi',\pi),
\end{equation}
with
$$
    \Delta J(\pi,\pi')
    = 
    \sum \limits_i \sum \limits_{(x^i, a^i) \in \mathcal{X}\times \mathcal{A}} (\rho^i_n(x^i, a^i)-{\rho'}^i_n(x^i, a^i))r^i(x^i, a^i, \mu_n),
$$
and
$$
    \Delta J(\pi',\pi)
    = 
    \sum \limits_i \sum \limits_{(x^i, a^i) \in \mathcal{X}\times \mathcal{A}} ({\rho'}^i_n(x^i, a^i) - \rho^i_n(x^i, a^i))r^i(x^i, a^i, {\mu'}_n).
$$
From here, we deduce~\eqref{Weak_Monotonicity}. Similarly, the strictly weak monotonicity implies a strict inequality in~\eqref{Weak_Monotonicity}.

\end{proof}

\section{Strictly weak monotonicity implies uniqueness}
\label{app:swm-uniqueness}

\begin{proof}[Proof of Lemma~\ref{lem:swm-uniqness}]

Consider a strictly weakly monotone game. For the sake of contradiction, assume that there exist two different Nash equilibria, say $\pi,\pi'$.

Proceeding as in the proof of Lemma~\ref{lem:wmon-tildeM}, we obtain~\eqref{eq:proof-uniqueness-ineq} with a strict inequality. 

Note that $\Delta J(\pi,\pi')$ corresponds to the difference between the reward of a typical player following $\pi$ when the population follows $\pi$ and the reward of a typical player following $\pi'$ when the population still follows $\pi$, and vice versa for $\Delta J(\pi',\pi)$. Moreover, $\pi,\pi'$ are Nash equilibria, so we deduce that these two terms are non-negative, which yields a contradiction with~\eqref{eq:proof-uniqueness-ineq}.

\end{proof}

\newpage

\section{Online Mirror descent convergence}
\label{proof_lyapunov}

\begin{proof}[Proof of Theorem~\ref{thm_convergence}]

Let $\Xi$ be defined as :
\begin{align*}
    \Xi(\pi^*,\pi) 
    &:=\sum_{i=1}^{N_p} \sum \limits_{n=0}^N \sum \limits_{x^i \in \mathcal{X}} \mu^{i,\pi^*}_n(x^i)[D_{h}(\pi^{i,*}_n(x^i,\cdot), \pi^i_n(x^i,\cdot))] .
\end{align*}
We pick $\pi\in\Delta \mathcal{A}$. If $\Delta J(\pi, \pi^*)  + \tilde {\mathcal{M}}(\pi, \pi^*) = 0$ then, $\tilde {\mathcal{M}}(\pi, \pi^*)=0$ and we can deduce that $\mu^{\pi}= \mu^{\pi^*}$. This implies that $\pi$ is a Nash as  $\pi$ and $\pi^*$ share the same distribution and thus the reward of a best response against $\pi$ or $\pi^*$ will be the same.

Let us suppose now that $\pi$ is a Nash and $\Delta J(\pi, \pi^*)  + \tilde {\mathcal{M}}(\pi, \pi^*) < 0$, then $\sum \limits_{i=1}^{N_p} J^i(\pi^i, \mu^{\pi}) - J^i(\pi^{i,*},\mu^{\pi})<0$ meaning that there exists an $i$ such that $J^i(\pi^i, \mu^{\pi}) - J^i(\pi^{i,*},\mu^{\pi})<0$. But as $\pi$ is a Nash, for all $\pi', i$, we have $J^i(\pi^i, \mu^{\pi}) - J^i({\pi'}^i,\mu^{\pi})\geq 0$ which is a contradiction.

Hence, if $\Delta J(\pi, \pi^*)  + \tilde {\mathcal{M}}(\pi, \pi^*) < 0$ then $\pi$ is not a Nash.

This proves that the Bregman divergence $\min\limits_{\pi^*\in \textrm{Nash}}\Xi(\pi^*,.)$ is a strict Lyapunov function of the CTOMD system. Hereby, $\pi_t$ converges to the set of Nash equilibria.
\end{proof}

Related to the Hypothesis in Theorem~\ref{thm_convergence}, we can show the following:
\begin{lemma}
If a \MPMFG satisfies $\tilde {\mathcal{M}}(\pi, \pi')<0$ if $\mu^{\pi}\neq \mu^{\pi'}$ and 0 otherwise, then there is at most one Nash equilibrium distribution. 
\end{lemma}
Note that uniqueness of the equilibrium distribution does not imply uniqueness of the equilibrium policy. This implication holds however under extra assumptions (e.g., some kind of strict convexity of the cost function).
\begin{proof}
Consider a \MPMFG satisfying the assumption. Consider two Nash equilibria, say $\pi,\pi'$. For the sake of contradiction, assume that they generate two different distributions $\mu^{\pi}, \mu^{\pi'}$. We have:
\begin{align*}
    0 &> 
    \tilde{\mathcal{M}}(\pi,\pi') 
    \\
    &= \sum \limits_{i=1}^{N_p}\big[J^i(\pi^i, \mu^{\pi}) +  J^i({\pi'}^i, \mu^{\pi'})
    - J^i(\pi^i, \mu^{\pi'}) - J^i({\pi'}^i, \mu^{\pi}) \big]
    \\
    &= \sum \limits_{i=1}^{N_p}\big[J^i(\pi^i, \mu^{\pi})  - J^i({\pi'}^i, \mu^{\pi}) \big]
    + \sum \limits_{i=1}^{N_p}\big[ J^i({\pi'}^i, \mu^{\pi'})
    - J^i(\pi^i, \mu^{\pi'}) \big]
\end{align*}
where both terms are non-negative because $\pi$ and $\pi'$ are Nash equilibria. Hence we must have $\mu^{\pi}=\mu^{\pi'}$.
\end{proof}

\newpage

\section{Numerical Experiments} \label{app:expe}

\subsection{Garnet}
\label{subappx:garnet}

\begin{figure*}[htbp]
  \centering
  \begin{subfigure}{1.0\textwidth}
    \centering
    \includegraphics[width=1\linewidth]{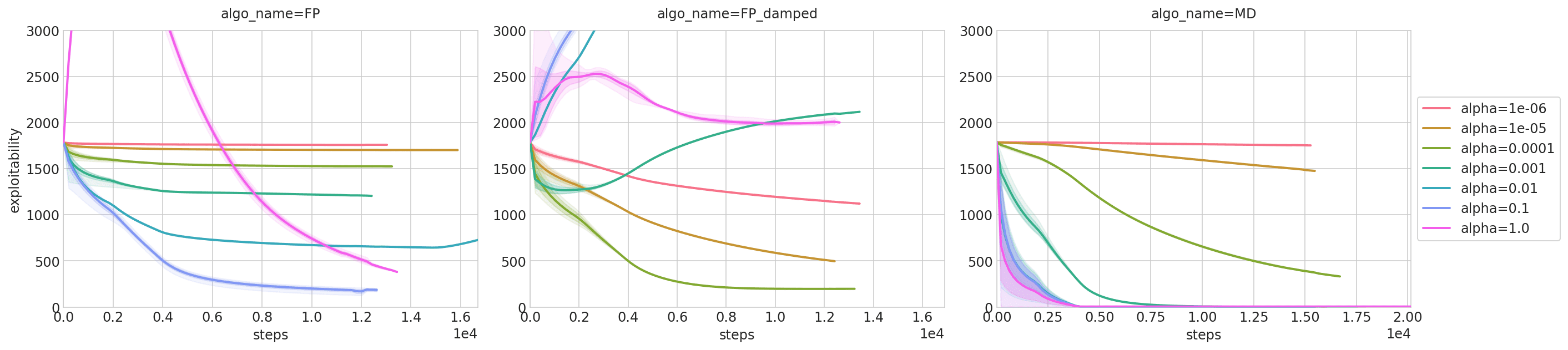}
    \vspace{-2em}
    \caption{5 garnet sampled with param $n_x=2000$, $n_a=20$, $t=2000$, $s_f=10$} 
    \label{garnet_b} 
  \end{subfigure}
  \begin{subfigure}{1.0\textwidth}
    \centering
    \includegraphics[width=1\linewidth]{images/images_compressed/garnet_ns20000_t2000_na10_sf10.png}
    \vspace{-2em}
    \caption{5 garnet sampled with param $n_x=20000$, $n_a=10$, $t=2000$, $s_f=10$} 
    \label{garnet_c}
  \end{subfigure}
  \begin{subfigure}{1.0\textwidth}
    \centering
    \includegraphics[width=1\linewidth]{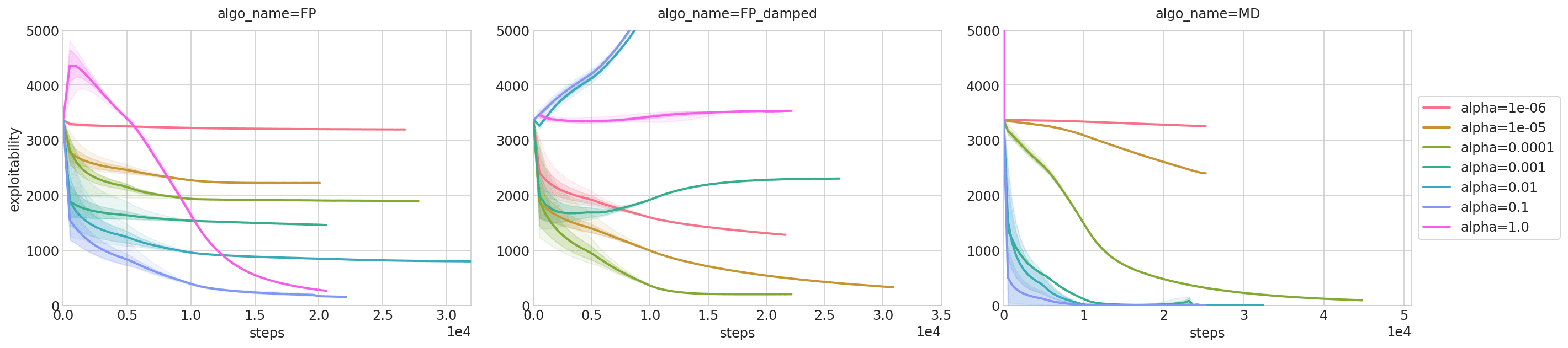}
    \vspace{-2em}
    \caption{5 garnet sampled with param $n_x=2000$, $n_a=10$, $t=2000$, $s_f=10$} 
    \vspace{-1em}
    \label{garnet_d}
  \end{subfigure}
  \caption{Garnet Experiments performances}
  \label{garnet_plot} 
\end{figure*}

\newpage 

\subsection{Building experiment}

\subsubsection{Building experiment performances}

\begin{figure*}[htbp]
    \centering
    \includegraphics[width=1\linewidth]{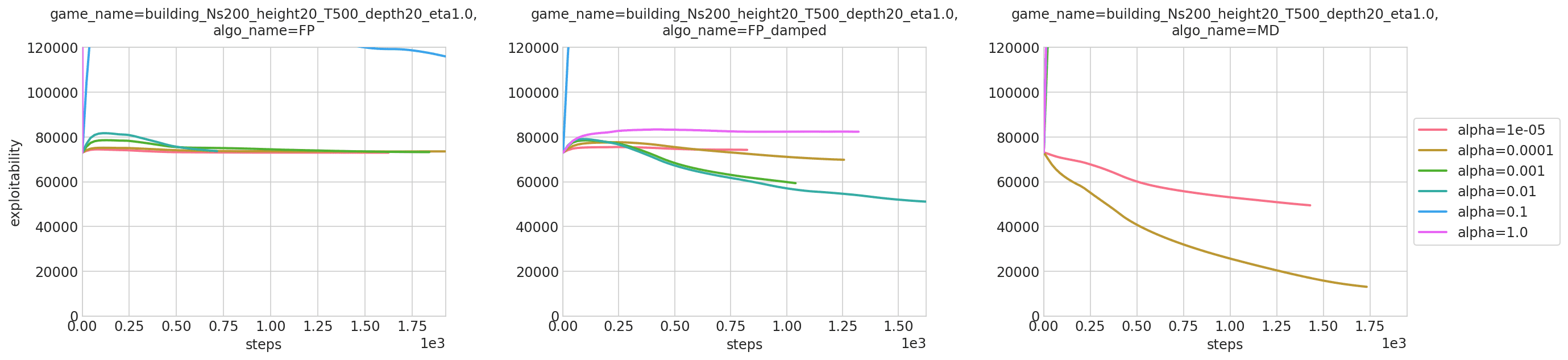}
    \vspace{-2em}
  \caption{Building Experiment performances}
  \label{building} 
\end{figure*}

\subsubsection{Building experiment solution}

The full building evacuation dynamics over the 20 floors is presented in Figure \ref{building_plot_20_storey} below.
\newpage 

\begin{figure*}[htbp]
  \centering
  \begin{subfigure}{0.49\textwidth}
    \centering
    \includegraphics[width=1.0\linewidth]{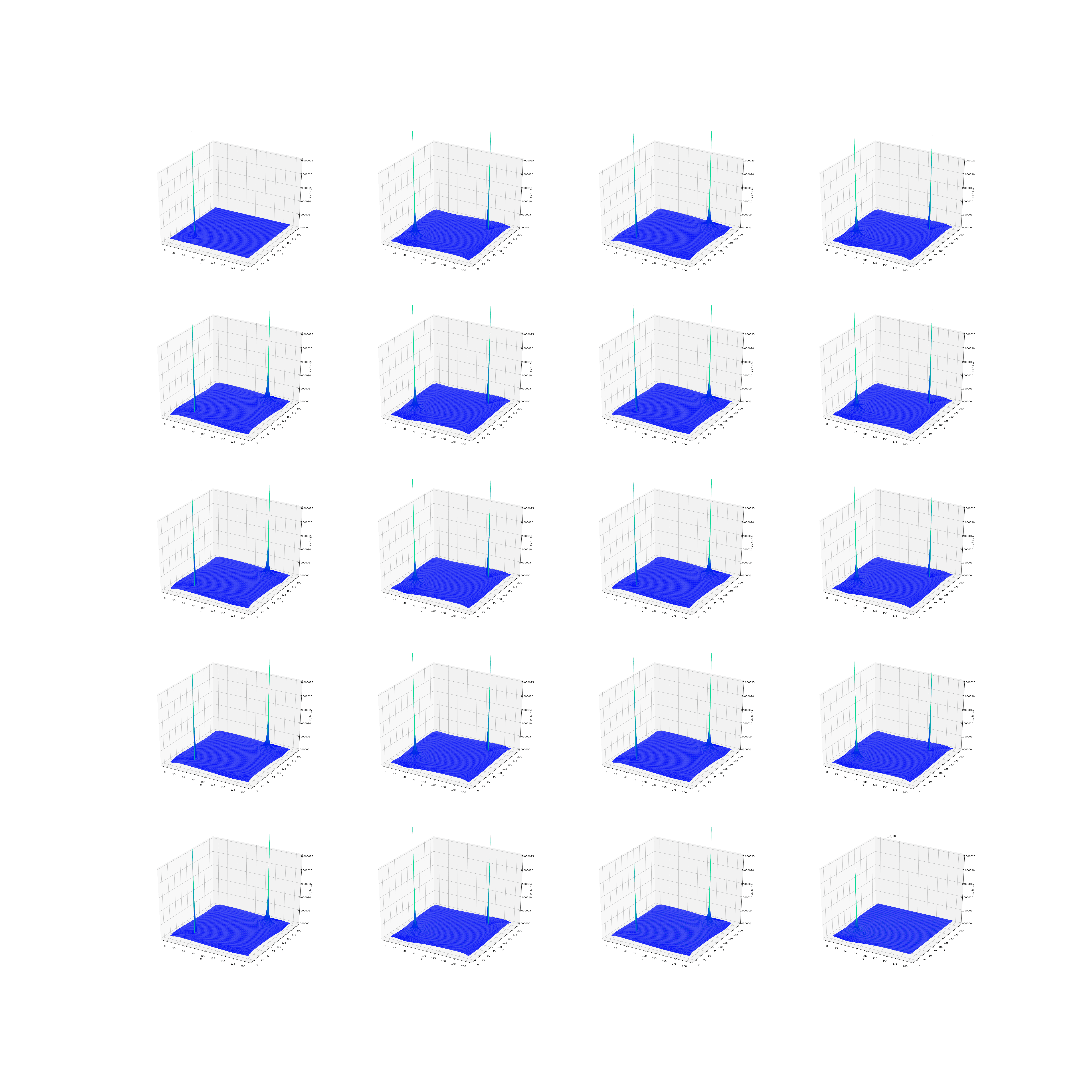}
        \vspace{-2em}
    \caption{After a few timesteps} 
    \vspace{1em}
    \label{build_2} 
  \end{subfigure}
  \begin{subfigure}{0.49\textwidth}
    \centering
    \includegraphics[width=1.0\linewidth]{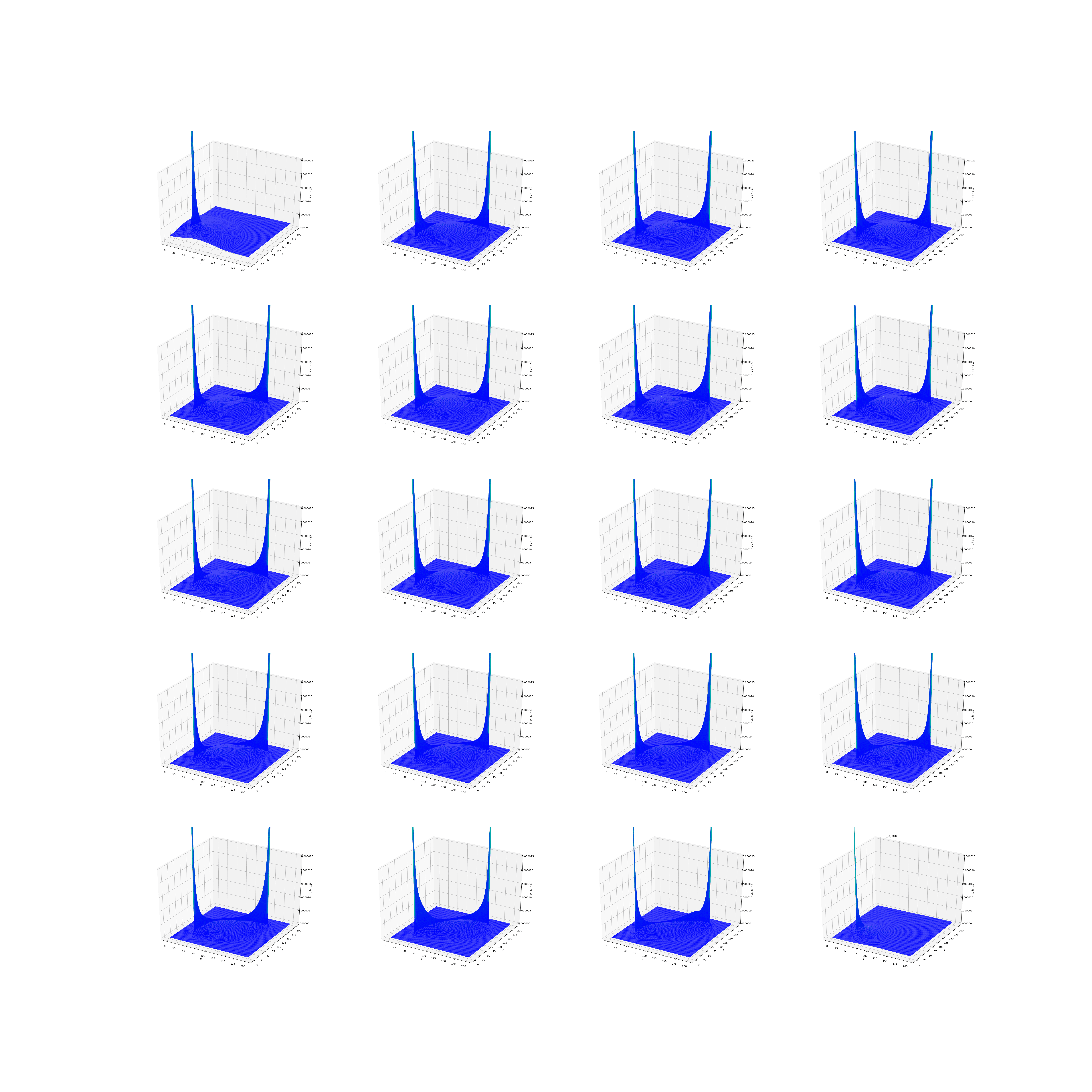}
    \vspace{-2em}
    \caption{Intermediate time} 
    \vspace{1em}
    \label{build_4}
  \end{subfigure}
  \begin{subfigure}{1.0\textwidth}
    \centering
    \includegraphics[width=0.49\linewidth]{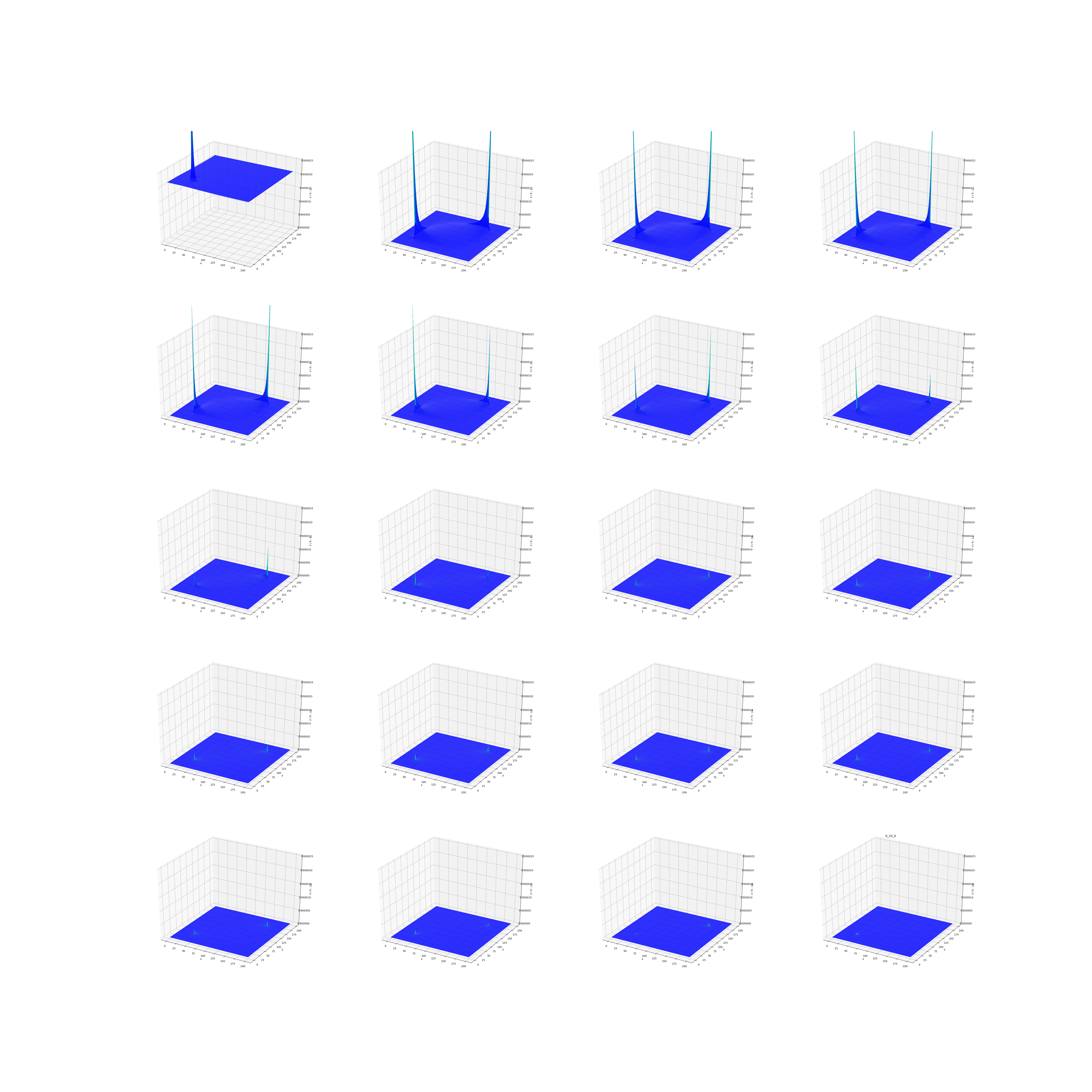}
    \vspace{-2em}
    \caption{Almost at arrival timestep} 
    \vspace{1em}
    \label{build_3}
  \end{subfigure}
  \caption{Building Experiment solution (ground floor on the upper left corner)}
  \label{building_plot_20_storey} 
\end{figure*}

\newpage

\subsection{Crowd motion with randomly shifted point of interest}

In this section, we discuss how to extend our results to the case of multi-population MFGs with common noise. 
In the example of Section~\ref{sec:crowd-random-shift}, the common noise corresponds to the geographical shifts of the point of interest.

The action space and the state space are the same but the dynamics and the reward are affected by a common noise sequence $\{\xi_n\}_{0\leq n\leq N}$. We denote $\Xi_n := \{\xi_k\}_{0\leq k < n} = \Xi_{n-1}.\xi_{n-1} $ the concatenation of the sequence $\Xi_{n-1}$ and the new noise $\xi_{n-1}$. By convention, we denote by $\Xi_{0}$ the empty sequence $\{\}$. $|\Xi_n| = n$ represents the total length of the sequence. The distribution of $\xi_{n}$ given the past sequence $\Xi_n$ is denoted by $P(.|\Xi_{n})$. 
Here, $\xi$ plays the role of a source or randomness which affects both the reward $r(x, a, \mu, \xi)$ and the probability transition function  $p(x'|x, a, \xi)$. It appears on top of the idiosyncratic randomness affecting each player. Policies and population distributions are now functions of the common noise and denoted respectively by $\pi^i_n(a|x,\Xi)$ and $\mu^i_{n}(x| \Xi)$ for population $i$. We will sometimes simply write $\pi^i_{n| \Xi}(a|x)$ and $\mu^i_{n| \Xi}(x)$.  Notice that the common noise is shared by all populations (we could also, with a slight modification, consider noises which are common to players of a given population and not shared with other populations). The $Q$ function of the $i$-th population now satisfies the following backward equation:
\begin{align*}
    &Q^{i,\pi^i, \mu}_N(x^i, a^i | \Xi_N) = r^i(x^i, a^i, \mu_{N | \Xi_N}, \xi_N)\\
    &Q^{i,\pi^i, \mu}_{n-1}(x^i, a^i | \Xi_{n-1}) = \sum_{\xi} P(\xi_{n-1} = \xi | \Xi_{n-1}) \Big[r^i(x^i, a^i, \mu_{n-1|\Xi_{n-1}}, \xi) \\
    &\qquad\qquad + \sum \limits_{{x'}^i\in \mathcal{X}} p({x'}^i|x^i, a^i, \xi) \mathbb{E}_{b^i\sim\pi^i_n(.|{x'}^i, \Xi_{n-1}.\xi)}\Big[Q^{i,\pi^i, \mu}_{n}(x^i, b^i | \Xi_{n-1}.\xi)\Big] \Big].
\end{align*}

For each population, the evolution of the distribution is conditioned on the realization of the common noise. It satisfies the forward equation: for all $x^i \in \mathcal{X}, \mu^{i,\pi^i}_{0|\Xi_0}(x)=\mu^{i}_0(x)$
and for all ${x'}^i \in \mathcal{X},$ 
\begin{equation*}
\mu^{i,\pi^i}_{n+1|\Xi_{n}.\xi_{n}}({x'}^i) = \sum \limits_{(x^i, a^i) \in \mathcal{X}\times \mathcal{A}} \pi^i_n(a^i|x^i,\Xi_n)p({x'}^i|x^i,a^i,\xi_{n})\mu^{i, \pi^i}_{n|\Xi_n}(x^i)
\end{equation*}
for $n \le N-1$. We denote $\mu^{\pi}=(\mu^{i,\pi^i})_{i \in \{1,\dots,N_p\}}$.

The expected total reward for a representative player of population $i$ using policy $\pi^i$ and facing the crowd behavior given by $\mu$ is:
\begin{align*}
    J^i(\pi^i, \mu) &= \mathbb{E}\Big[\sum \limits_{n=0}^N r^i(x^i_n, a^i_n, \mu_{n|\Xi_n},\xi_n) \;\Big|\; x^i_0 \sim \mu^i_0, a^i_n \sim \pi^i_n(.|x^i_n, \xi_n), x^i_{n+1}\sim p(.|x^i_n, a^i_n, \xi_{n}), \xi_{n} \sim P(.|\Xi_{n})\Big].
\end{align*}

\textbf{Continuous time Online Mirror Descent for \MPMFG{}s with common noise:}

In this setting, the Continuous Time Online Mirror Descent (CTOMD) algorithm is defined as: for all $i\in \{1,\dots,N_p\}, n\in \{0,\dots,N\}$, $y^i_{n,0} = 0$, and for all $t \in \mathbb{R}_+$,
\begin{align}
    y^i_{n,t}(x^i,a^i|\Xi_n) &= \int \limits_{0}^{t} Q^{i,\pi^i_s, \mu^{\pi_s}}_n(x^i,a^i|\Xi_n) ds,
    \label{eq:OMD-y-common}
    \\
    \pi^i_{n,t}(.|x^i,\Xi_n) &= \Gamma(y^i_{n,t}(x^i,.|\Xi_n)).
    \label{eq:OMD-pi-common}
\end{align}

Our theoretical results naturally extend to this setting by following similar arguments as the ones in~\cite{perrin2020fictitious}.

\newpage 
\subsection{Multi-population}\label{appx:sec_multipop_expes}

\subsubsection{Monotony of the multi-population reward}
\label{appx:subsec_proof_monotone_multi_reward}

We prove rigorously that the MP-MFG reward is monotone. As $\tilde r(x,a) = 0$, the separability condition is trivially verified. Furthermore, we have:

\begin{equation*}
\begin{aligned}
    \sum_i \sum_{x \in \mathcal{X}} (\mu^i(x) - \mu'^i(x))(\hat r^i(x,\mu) - \hat r^i(x,\mu')) = & \sum_i \sum_{x \in \mathcal{X}} (\mu^i(x) - \mu'^i(x))(-\log(\mu^i(x)) + \\
    & \sum_{j\neq i}\mu^j(x)\bar r^{i,j}(x)+\log(\mu^i(x)) - \sum_{j\neq i}\mu^j(x)\bar r^{i,j}(x)) \\
    = & \underbrace{\sum_i \sum_{x \in \mathcal{X}} (\mu^i(x) - \mu'^i(x))(-\log(\mu^i(x)) + \log(\mu'^i(x)))}_{\textrm{(1)}} + \\
    & \underbrace{\sum_i \sum_{x \in \mathcal{X}} (\mu^i(x) - \mu'^i(x)) (\sum_{j\neq i}\mu^j(x)\bar r^{i,j}(x) - \sum_{j\neq i}\mu^j(x)\bar r^{i,j}(x))}_{\textrm{(2)}};
\end{aligned}
\end{equation*}
where we have:
\begin{itemize}
    \item $(1) \leq 0$ because $\forall x, \forall i$, $(\mu^i(x) - \mu'^i(x)) (-\log(\mu^i(x)) + \log(\mu'^i(x))) \leq 0$ as $\log$ is an increasing function;
    \item $(2) = 0$ because $\forall i, \forall j, i\neq j, \bar{r}^{i,j}(x) = - \bar{r}^{j, i}(x)$.
\end{itemize}

Thus,
\begin{equation*}
    \sum_i \sum_{x \in \mathcal{X}} (\mu^i(x) - \mu'^i(x))(\hat r^i(x,\mu) - \hat r^i(x,\mu')) \leq 0.
\end{equation*}

\subsubsection{Multi-population performances}\label{appx:subsec_multipop_medium}

The performances of Fictitious Play and Online Mirror Descent for the multi-population chasing Mean Field Game with different field topologies and initial distribution are presented in Figure \ref{multi_pop_medium}.

\begin{figure*}[htbp]
  \centering
  \begin{subfigure}{1.0\textwidth}
    \centering
    \includegraphics[width=0.85\linewidth]{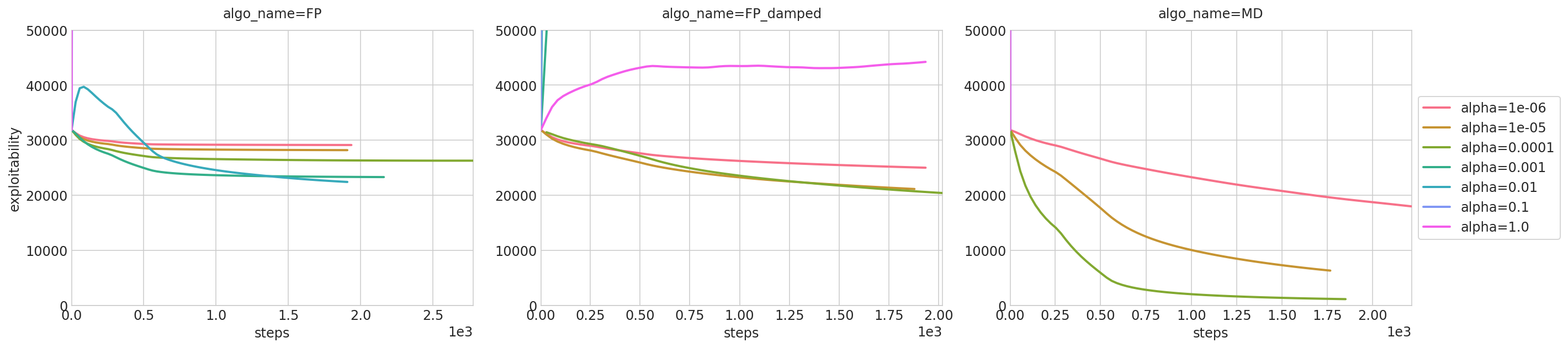}
    \caption{Torus topology and corner initialization} 
    \label{torus_corner} 
  \end{subfigure}
  \begin{subfigure}{1.0\textwidth}
    \centering
    \includegraphics[width=0.85\linewidth]{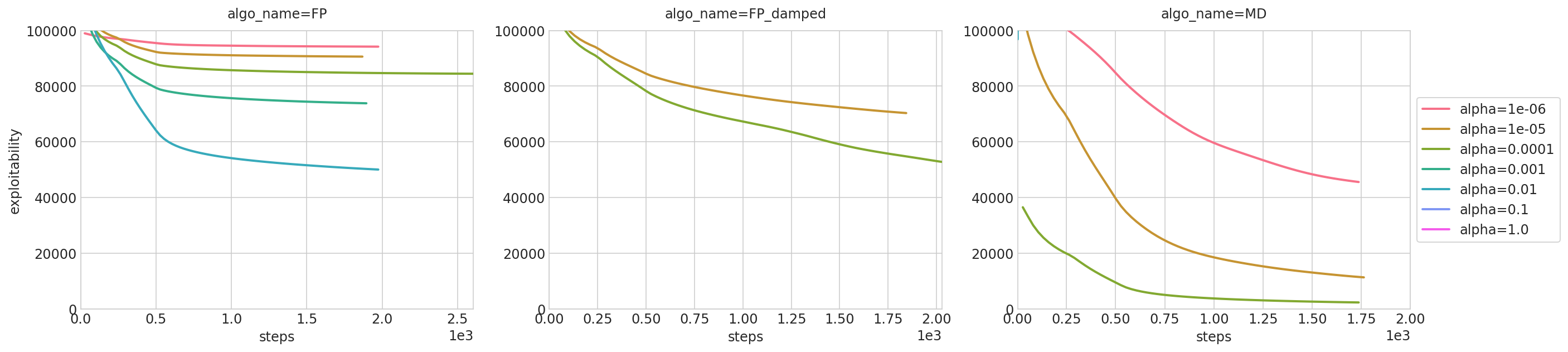}
    \caption{Square topology and corner initialization} 
    \label{square_corner} 
  \end{subfigure}
  \begin{subfigure}{1.0\textwidth}
    \centering
    \includegraphics[width=0.85\linewidth]{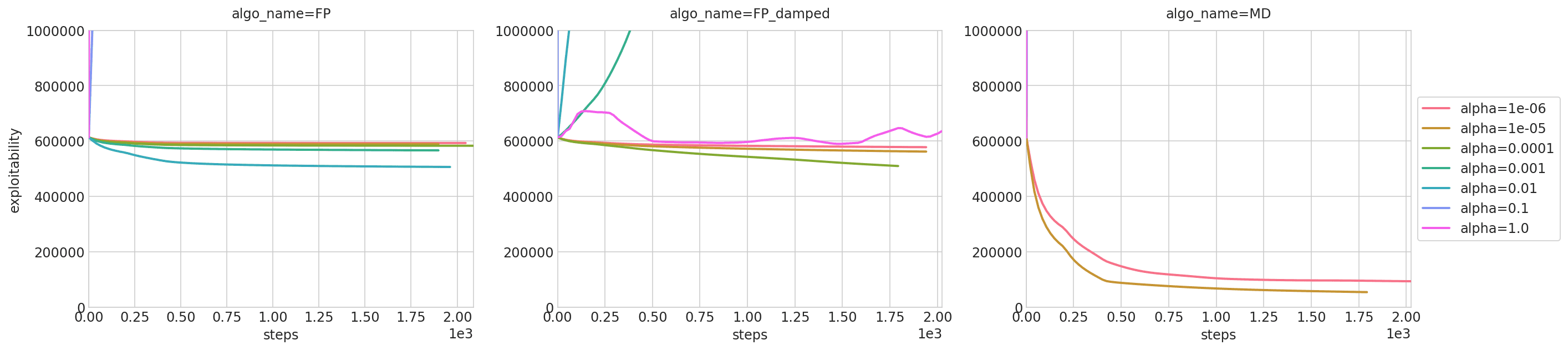}
    \caption{Donut topology and corner initialization} 
    \label{donut_corner} 
  \end{subfigure}
  \begin{subfigure}{1.0\textwidth}
    \centering
    \includegraphics[width=0.85\linewidth]{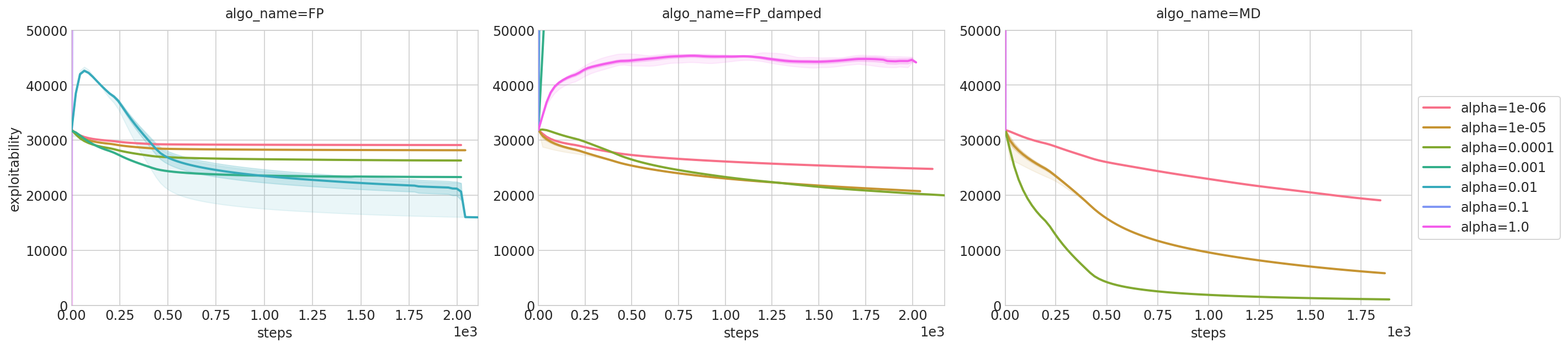}
    \caption{Torus topology and random initialization} 
    \label{torus_random} 
  \end{subfigure}
  \begin{subfigure}{1.0\textwidth}
    \centering
    \includegraphics[width=0.85\linewidth]{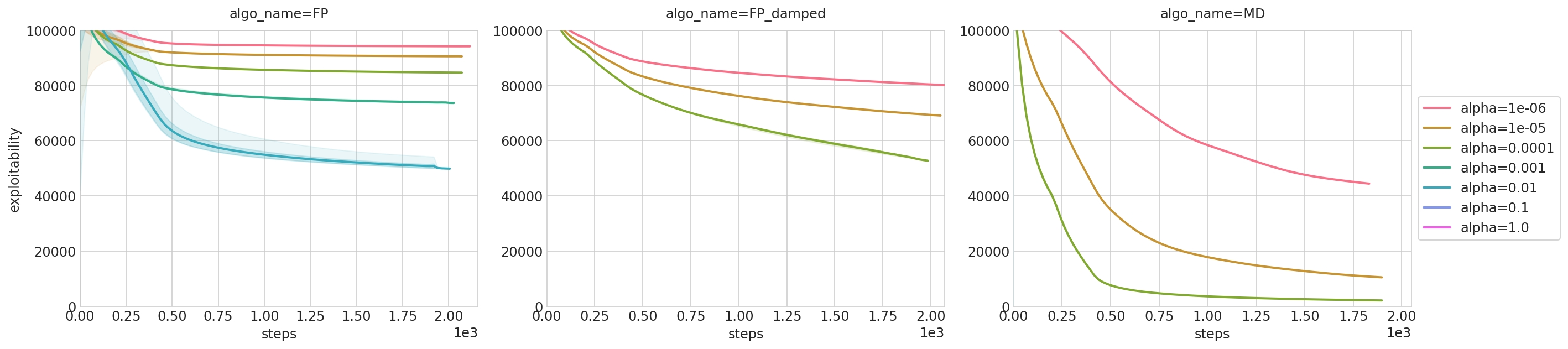}
    \caption{Square topology and random initialization} 
    \label{square_random} 
  \end{subfigure}
  \begin{subfigure}{1.0\textwidth}
    \centering
    \includegraphics[width=0.85\linewidth]{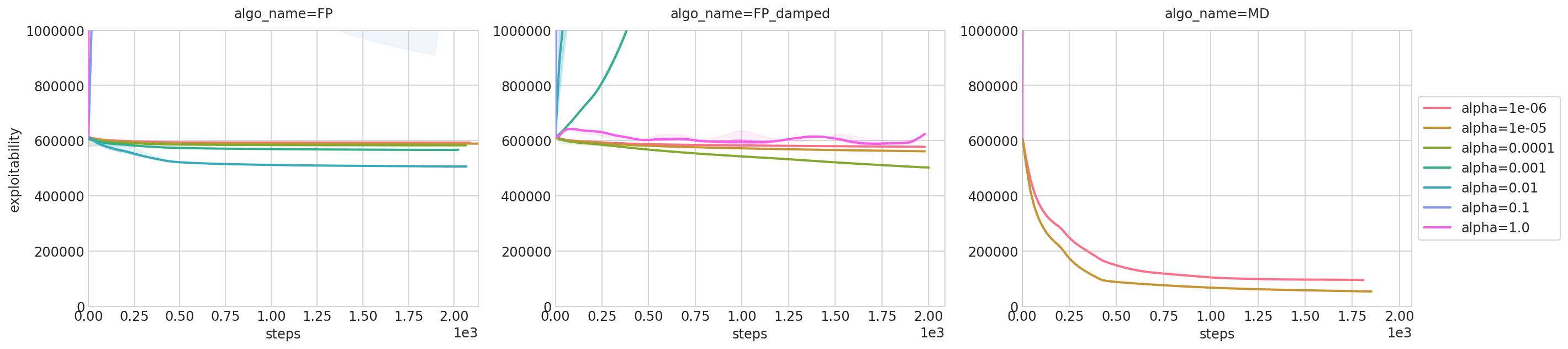}
    \caption{Donut topology and random initialization} 
    \label{donut_random} 
  \end{subfigure}
  \caption{Multi-population experiments, performances with different topologies}
  \label{multi_pop_medium} 
\end{figure*}
\end{document}